%% file: Paper.tex
\documentclass{article}
\usepackage[final]{neurips_2023}

\usepackage[utf8]{inputenc} 
\usepackage[T1]{fontenc}    
\usepackage{url}            
\usepackage{booktabs}       
\usepackage{amsfonts}       
\usepackage{nicefrac}       
\usepackage{microtype}      

\usepackage{algorithm}
\usepackage{algorithmicx}
\usepackage[noend]{algpseudocode}
\usepackage{amsmath}
\usepackage{amssymb}
\usepackage{amsthm}
\usepackage{bbm}
\usepackage{bm}
\usepackage{color}
\usepackage{dirtytalk}
\usepackage{dsfont}
\usepackage{enumerate}
\usepackage{graphicx}
\usepackage{listings}
\usepackage{mathtools}
\usepackage{subfigure}
\usepackage{times}
\usepackage{xspace}

\usepackage[usenames,dvipsnames]{xcolor}
\usepackage[bookmarks=false]{hyperref}
\hypersetup{
  pdffitwindow=true,
  pdfstartview={FitH},
  pdfnewwindow=true,
  colorlinks,
  linktocpage=true,
  linkcolor=Green,
  urlcolor=Green,
  citecolor=Green
}
\usepackage[capitalize,noabbrev]{cleveref}

\usepackage[textsize=tiny]{todonotes}

\usepackage{thmtools}
\declaretheorem[name=Theorem,refname={Theorem,Theorems},Refname={Theorem,Theorems}]{theorem}
\declaretheorem[name=Lemma,refname={Lemma,Lemmas},Refname={Lemma,Lemmas},sibling=theorem]{lemma}
\declaretheorem[name=Corollary,refname={Corollary,Corollaries},Refname={Corollary,Corollaries},sibling=theorem]{corollary}

\newcommand{\cA}{\mathcal{A}}

\newcommand{\cN}{\mathcal{N}}

\newcommand{\realset}{\mathbb{R}}

\newcommand{\E}[1]{\mathbb{E} \left[#1\right]}
\newcommand{\condE}[2]{\mathbb{E} \left[#1 \,\middle|\, #2\right]}

\newcommand{\condprob}[2]{\mathbb{P} \left(#1 \,\middle|\, #2\right)}

\newcommand*\dif{\mathop{}\!\mathrm{d}}

\newcommand{\I}[1]{\mathds{1} \! \left\{#1\right\}}

\newcommand{\normw}[2]{\|#1\|_{#2}}

\newcommand{\set}[1]{\left\{#1\right\}}

\newcommand{\T}{^\top}

\DeclareMathOperator*{\argmax}{arg\,max\,}

\let\det\relax
\DeclareMathOperator{\det}{det}

\let\trace\relax
\DeclareMathOperator{\trace}{tr}
\mathchardef\mhyphen="2D

\newcommand{\bayesucb}{\ensuremath{\tt BayesUCB}\xspace}

\newcommand{\linucb}{\ensuremath{\tt LinUCB}\xspace}

\newcommand{\ucb}{\ensuremath{\tt UCB1}\xspace}

\title{Finite-Time Logarithmic Bayes Regret Upper Bounds}

\author{
  Alexia Atsidakou \\
  University of Texas, Austin
  \And
  Branislav Kveton \\
  AWS AI Labs$^*$
  \And
  Sumeet Katariya \\
  Amazon
  \And
  Constantine Caramanis \\
  University of Texas, Austin
  \And
  Sujay Sanghavi \\
  University of Texas, Austin / Amazon
}

\begin{document}

\maketitle

\begin{abstract}
We derive the first finite-time logarithmic Bayes regret upper bounds for Bayesian bandits. In a multi-armed bandit, we obtain $O(c_\Delta \log n)$ and $O(c_h \log^2 n)$ upper bounds for an upper confidence bound algorithm, where $c_h$ and $c_\Delta$ are constants depending on the prior distribution and the gaps of bandit instances sampled from it, respectively. The latter bound asymptotically matches the lower bound of Lai (1987). Our proofs are a major technical departure from prior works, while being simple and general. To show the generality of our techniques, we apply them to linear bandits. Our results provide insights on the value of prior in the Bayesian setting, both in the objective and as a side information given to the learner. They significantly improve upon existing $\tilde{O}(\sqrt{n})$ bounds, which have become standard in the literature despite the logarithmic lower bound of Lai (1987).
\end{abstract}

\input{Introduction}

\input{Setting}

\input{Algorithm}

\input{RegretBounds}

\input{PriorWork}

\input{Experiments}

\input{Conclusions}

\bibliographystyle{plainnat}
\bibliography{References}

\clearpage
\onecolumn
\appendix

\input{Proofs}

\input{Corollary}

\input{GapFreeLinUCB}

\input{Bakeoff}

\input{InfoTheoryNote}

\end{document}

%% file: Introduction.tex
\section{Introduction}
\label{sec:introduction}

A \emph{stochastic multi-armed bandit} \citep{lai85asymptotically,auer02finitetime,lattimore19bandit} is an online learning problem where a \emph{learner} sequentially interacts with an environment over $n$ rounds. In each round, the learner takes an \emph{action} and receives its \emph{stochastic reward}. The goal of the learner is to maximize its expected cumulative reward over $n$ rounds. The mean rewards of the actions are unknown \emph{a priori} but can be learned by taking the actions. Therefore, the learner faces the \emph{exploration-exploitation dilemma}: \emph{explore}, and learn more about the actions; or \emph{exploit}, and take the action with the highest estimated reward. Bandits have been successfully applied to problems where uncertainty modeling and subsequent adaptation are beneficial. One example are recommender systems \citep{li10contextual,zhao13interactive,kawale15efficient,li16collaborative}, where the actions are recommended items and their rewards are clicks. Another example is hyper-parameter optimization \citep{li18hyperband}, where the actions are values of the optimized parameters and their reward is the optimized metric.

Cumulative regret minimization in stochastic bandits has been traditionally studied in two settings: frequentist \citep{lai85asymptotically,auer02finitetime,abbasi-yadkori11improved} and Bayesian \citep{gittins79bandit,tsitsiklis94short,lai87adaptive,russo14learning,russo18tutorial}. In the frequentist setting, the learner minimizes the regret with respect to a fixed unknown bandit instance. In the Bayesian setting, the learner minimizes the average regret with respect to bandit instances drawn from a prior distribution. The instance is unknown but the learner knows its prior distribution. The Bayesian setting allows surprisingly simple and insightful analyses of Thompson sampling. One fundamental result in this setting is that linear Thompson sampling \citep{russo14learning} has a comparable regret bound to \linucb in the frequentist setting \citep{abbasi-yadkori11improved,agrawal13thompson,abeille17linear}. Moreover, many recent meta- and multi-task bandit works \citep{bastani19meta,kveton21metathompson,basu21noregrets,simchowitz21bayesian,wang21multitask,hong22hierarchical,aouali23mixedeffect} adopt the Bayes regret to analyze the stochastic structure of their problems, that the bandit tasks are similar because their parameters are sampled i.i.d.\ from a task distribution.

Many bandit algorithms have frequentist regret bounds that match a lower bound. As an example, in a $K$-armed bandit with the minimum gap $\Delta$ and horizon $n$, the gap-dependent $O(K \Delta^{-1} \log n)$ regret bound of \ucb \citep{auer02finitetime} matches the gap-dependent $\Omega(K \Delta^{-1} \log n)$ lower bound of \citet{lai85asymptotically}. Moreover, the gap-free $\tilde{O}(\sqrt{K n})$ regret bound of \ucb matches, up to logarithmic factors, the gap-free $\Omega(\sqrt{K n})$ lower bound of \citet{auer95gambling}. The extra logarithmic factor in the $\tilde{O}(\sqrt{K n})$ bound can be eliminated by modifying \ucb \citep{audibert09minimax}. In contrast, and despite the popularity of the model, matching upper and lower bounds mostly do not exist in the Bayesian setting. Specifically, \citet{lai87adaptive} proved \emph{asymptotic} $c_h \log^2 n$ upper and lower bounds, where $c_h$ is a prior-dependent constant. However, all recent Bayes regret bounds are $\tilde{O}(\sqrt{n})$ \citep{russo14learning,russo16information,lu19informationtheoretic,hong20latent,kveton21metathompson}. This leaves open the question of finite-time logarithmic regret bounds in the Bayesian setting.

In this work, we answer this question positively and make the following contributions: 
\begin{enumerate}
  \item We derive the first finite-time logarithmic Bayes regret upper bounds for a Bayesian \emph{upper confidence bound (UCB)} algorithm. The bounds are $O(c_\Delta \log n)$ and $O(c_h \log^2 n)$, where $c_h$ and $c_\Delta$ are constants depending on the prior distribution $h$ and the gaps of random bandit instances sampled from $h$, respectively. The latter matches the lower bound of \citet{lai87adaptive} asymptotically. When compared to prior $\tilde{O}(\sqrt{n})$ bounds, we better characterize low-regret regimes, where the random gaps are large.
  \item To show the value of prior as a side information, we also derive a finite-time logarithmic Bayes regret upper bound for a frequentist UCB algorithm. The bound changes only little as the prior becomes more informative, while the regret bound for the Bayesian algorithm eventually goes to zero. The bounds match asymptotically when $n \to \infty$ and the prior is overtaken by data.
  \item To show the generality of our approach, we prove a $O(d \, c_\Delta \log^2 n)$ Bayes regret bound for a Bayesian linear bandit algorithm, where $d$ denotes the number of dimensions and $c_\Delta$ is a constant depending on random gaps. This bound also improves with a better prior.
  \item Our analyses are a major departure from all recent Bayesian bandit analyses, starting with \citet{russo14learning}. Roughly speaking, we first bound the regret in a fixed bandit instance, similarly to frequentist analyses, and then integrate out the random gap.
  \item We show the tightness of our bounds empirically and compare them to prior bounds. 
\end{enumerate}

This paper is organized as follows. In \cref{sec:setting}, we introduce the setting of Bayesian bandits. In \cref{sec:algorithm}, we present a Bayesian upper confidence bound algorithm called \bayesucb \citep{kaufmann12bayesian}. In \cref{sec:regret bounds}, we derive finite-time logarithmic Bayes regret bounds for \bayesucb, in both multi-armed and linear bandits. These are the first such bounds ever derived. In \cref{sec:prior work}, we compare our bounds to prior works and show that one matches an existing lower bound \citep{lai87adaptive} asymptotically. In \cref{sec:experiments}, we evaluate the bounds empirically. We conclude in \cref{sec:conclusions}.

\renewcommand{\thefootnote}{\fnsymbol{footnote}}
\footnotetext[1]{The work started at Amazon Search.}
\renewcommand{\thefootnote}{\arabic{footnote}}

%% file: Setting.tex
\section{Setting}
\label{sec:setting}

We start with introducing our notation. Random variables are capitalized, except for Greek letters like $\theta$. For any positive integer $n$, we define $[n] = \set{1, \dots, n}$. The indicator function is $\I{\cdot}$. The $i$-th entry of vector $v$ is $v_i$. If the vector is already indexed, such as $v_j$, we write $v_{j, i}$. We denote the maximum and minimum eigenvalues of matrix $M \in \realset^{d \times d}$ by $\lambda_1(M)$ and $\lambda_d(M)$, respectively.

Our setting is defined as follows. We have a \emph{multi-armed bandit} \citep{lai85asymptotically,lai87adaptive,auer02finitetime,abbasi-yadkori11improved} with an \emph{action set} $\cA$. Each \emph{action} $a \in \cA$ is associated with a \emph{reward distribution} $p_a(\cdot; \theta)$, which is parameterized by an unknown \emph{model parameter} $\theta$ shared by all actions. The learner interacts with the bandit instance for $n$ rounds indexed by $t \in [n]$. In each round $t$, it takes an \emph{action} $A_t \in \cA$ and observes its \emph{stochastic reward} $Y_t \sim p_{A_t}(\cdot; \theta)$. The rewards are sampled independently across the rounds. We denote the mean of $p_a(\cdot; \theta)$ by $\mu_a(\theta)$ and call it the \emph{mean reward} of action $a$. The optimal action is $A_* = \argmax_{a \in \cA} \mu_a(\theta)$ and its mean reward is $\mu_*(\theta) = \mu_{A_*}(\theta)$. For a fixed model parameter $\theta$, the $n$-round \emph{regret} of a policy is defined as
\begin{align*}
  R(n; \theta)
  = \condE{\sum_{t = 1}^n \mu_*(\theta) - \mu_{A_t}(\theta)}{\theta}\,,
\end{align*}
where the expectation is taken over both random observations $Y_t$ and actions $A_t$. The \emph{suboptimality gap} of action $a$ is $\Delta_a = \mu_*(\theta) - \mu_a(\theta)$ 
and the \emph{minimum gap} is $\Delta_{\min} = \min_{a \in \cA \setminus \set{A_*}} \Delta_a$.

Two settings are common in stochastic bandits. In the \emph{frequentist} setting \citep{lai85asymptotically,auer02finitetime,abbasi-yadkori11improved}, the learner has no additional information about $\theta$ and its objective is to minimize the worst-case regret for any bounded $\theta$. We study the \emph{Bayesian} setting \citep{gittins79bandit,lai87adaptive,russo14learning,russo18tutorial}, where the model parameter $\theta$ is drawn from a \emph{prior distribution} $h$ that is given to the learner as a side information. The goal of the learner is to minimize the $n$-round \emph{Bayes regret} $R(n) = \E{R(n; \theta)}$, where the expectation is taken over the random model parameter $\theta \sim h$. Note that $A_*$, $\Delta_a$, and $\Delta_{\min}$ are random because they depend on the random instance $\theta$.

%% file: Algorithm.tex
\section{Algorithm}
\label{sec:algorithm}

We study a Bayesian upper confidence bound algorithm called \bayesucb \citep{kaufmann12bayesian}. The algorithm was analyzed in the Bayesian setting by \citet{russo14learning}. The key idea in \bayesucb is to take the action with the highest UCB with respect to the posterior distribution of model parameter $\theta$. This differentiates it from frequentist algorithms, such as \ucb \citep{auer02finitetime} and \linucb \citep{abbasi-yadkori11improved}, where the UCBs are computed using a frequentist \emph{maximum likelihood estimate (MLE)} of the model parameter.

Let $H_t = (A_\ell, Y_\ell)_{\ell \in [t - 1]}$ be the \emph{history} of taken actions and their observed rewards up to round $t$. The \emph{Bayesian UCB} for the mean reward of action $a$ at round $t$ is
\begin{align}
  U_{t, a}
  = \mu_a(\hat{\theta}_t) + C_{t, a}\,,
  \label{eq:ucb}
\end{align}
where $\hat{\theta}_t$ is the \emph{posterior mean estimate} of $\theta$ at round $t$ and $C_{t, a}$ is a \emph{confidence interval width} for action $a$ at round $t$. The posterior distribution of model parameter $\theta$ is computed from a prior $h$ and history $H_t$ using Bayes' rule. The width is chosen so that $|\mu_a(\hat{\theta}_t) - \mu_a(\theta)| \leq C_{t, a}$ holds with a high probability conditioned on any history $H_t$. Technically speaking, $C_{t, a}$ is a half-width but we call it a width to simplify terminology.

\begin{algorithm}[t]
  \caption{\bayesucb}
  \label{alg:bayesucb}
  \begin{algorithmic}[1]
    \For{$t = 1, \dots, n$}
      \State Compute the posterior distribution of $\theta$ using prior $h$ and history $H_t$
      \For{each action $a \in \cA$}
        \State Compute $U_{t, a}$ according to \eqref{eq:ucb}
      \EndFor
      \State Take action $A_t \gets \argmax_{a \in \cA} U_{t, a}$ and observe its reward $Y_t$
    \EndFor
  \end{algorithmic}
\end{algorithm}

Our algorithm is presented in \cref{alg:bayesucb}. We instantiate it in a Gaussian bandit in \cref{sec:gaussian bandit}, in a Bernoulli bandit in \cref{sec:bernoulli bandit}, and in a linear bandit with Gaussian rewards in \cref{sec:linear bandit}. These settings are of practical interest because they lead to computationally-efficient implementations that can be analyzed due to closed-form posteriors \citep{lu19informationtheoretic,kveton21metathompson,basu21noregrets,wang21multitask,hong22hierarchical}. While we focus on deriving logarithmic Bayes regret bounds for \bayesucb, we believe that similar analyses can be done for Thompson sampling \citep{thompson33likelihood,chapelle11empirical,agrawal12analysis,agrawal13thompson,russo14learning,russo18tutorial}. This extension is non-trivial because a key step in our analysis is that the action with the highest UCB is taken (\cref{sec:technical novelty}).

\subsection{Gaussian Bandit}
\label{sec:gaussian bandit}

In a $K$-armed Gaussian bandit, the action set is $\cA = [K]$ and the model parameter is $\theta \in \realset^K$. Each action $a \in \cA$ has a Gaussian reward distribution, $p_a(\cdot; \theta) = \cN(\cdot; \theta_a, \sigma^2)$, where $\theta_a$ is its mean and $\sigma > 0$ is a known reward noise. Thus $\mu_a(\theta) = \theta_a$. The model parameter $\theta$ is drawn from a known Gaussian prior $h(\cdot)= \cN(\cdot; \mu_0, \sigma_0^2 I_K)$, where $\mu_0 \in \realset^K$ is a vector of prior means and $\sigma_0 > 0$ is a prior width.

The posterior distribution of the mean reward of action $a$ at round $t$ is $\cN(\cdot; \hat{\theta}_{t, a}, \hat{\sigma}_{t, a}^2)$, where
\begin{align*}
  \hat{\sigma}_{t, a}^2
  = (\sigma_0^{-2} + \sigma^{-2} N_{t, a})^{-1}
\end{align*}
is the posterior variance, $N_{t, a} = \sum_{\ell = 1}^{t - 1} \I{A_\ell = a}$ is the number of observations of action $a$ up to round $t$, and
\begin{align*}
  \hat{\theta}_{t, a}
  = \hat{\sigma}_{t, a}^2 \left(\sigma_0^{-2} \mu_{0, a} +
  \sigma^{-2} \sum_{\ell = 1}^{t - 1} \I{A_\ell = a} Y_\ell\right)
\end{align*}
is the posterior mean. This follows from a classic result, that the posterior distribution of the mean of a Gaussian random variable with a Gaussian prior is a Gaussian \citep{bishop06pattern}. The Bayesian UCB of action $a$ at round $t$ is $U_{t, a} = \hat{\theta}_{t, a} + C_{t, a}$, where $C_{t, a} = \sqrt{2 \hat{\sigma}_{t, a}^2 \log(1 / \delta)}$ is the confidence interval width and $\delta \in (0, 1)$ is a failure probability of the confidence interval.

\subsection{Bernoulli Bandit}
\label{sec:bernoulli bandit}

In a $K$-armed Bernoulli bandit, the action set is $\cA = [K]$ and the model parameter is $\theta \in \realset^K$. Each action $a \in \cA$ has a Bernoulli reward distribution, $p_a(\cdot; \theta) = \mathrm{Ber}(\cdot; \theta_a)$, where $\theta_a$ is its mean. Hence $\mu_a(\theta) = \theta_a$. Each parameter $\theta_a$ is drawn from a known prior $\mathrm{Beta}(\cdot; \alpha_a, \beta_a)$, where $\alpha_a > 0$ and $\beta_a > 0$ are positive and negative prior pseudo-counts, respectively.

The posterior distribution of the mean reward of action $a$ at round $t$ is $\mathrm{Beta}(\cdot; \alpha_{t, a}, \beta_{t, a})$, where
\begin{align*}
  \alpha_{t, a}
  = \alpha_a + \sum_{\ell = 1}^{t - 1} \I{A_\ell = a} Y_\ell\,, \quad
  \beta_{t, a}
  = \beta_a + \sum_{\ell = 1}^{t - 1} \I{A_\ell = a} (1 - Y_\ell)\,.
\end{align*}
This follows from a classic result, that the posterior distribution of the mean of a Bernoulli random variable with a beta prior is a beta distribution \citep{bishop06pattern}. The corresponding Bayesian UCB is $U_{t, a} = \hat{\theta}_{t, a} + C_{t, a}$, where
\begin{align*}
  \hat{\theta}_{t, a}
  = \frac{\alpha_{t, a}}{\alpha_{t, a} + \beta_{t, a}}\,, \quad
  C_{t, a}
  = \sqrt{\frac{\log(1 / \delta)}{2 (\alpha_{t, a} + \beta_{t, a} + 1)}}
  = \sqrt{\frac{\log(1 / \delta)}{2 (\alpha_a + \beta_a + N_{t, a} + 1)}}\,,
\end{align*}
denote the posterior mean and confidence interval width, respectively, of action $a$ at round $t$; and $\delta \in (0, 1)$ is a failure probability of the confidence interval. The confidence interval is derived using the fact that $\mathrm{Beta}(\cdot; \alpha_{t, a}, \beta_{t, a})$ is a sub-Gaussian distribution with variance proxy $\frac{1}{4 (\alpha_a + \beta_a + N_{t, a} + 1)}$ \citep{marchal17subgaussianity}.

\subsection{Linear Bandit with Gaussian Rewards}
\label{sec:linear bandit}

We also study linear bandits \citep{dani08stochastic,abbasi-yadkori11improved} with a finite number of actions $\cA \subseteq \realset^d$ in $d$ dimensions. The model parameter is $\theta \in \realset^d$. All actions $a \in \cA$ have Gaussian reward distributions, $p_a(\cdot; \theta) = \cN(\cdot; a\T \theta, \sigma^2)$, where $\sigma > 0$ is a known reward noise. Therefore, the mean reward of action $a$ is $\mu_a(\theta) = a\T \theta$. The parameter $\theta$ is drawn from a known multivariate Gaussian prior $h(\cdot)= \cN(\cdot; \theta_0, \Sigma_0)$, where $\theta_0 \in \realset^d$ is its mean and $\Sigma_0 \in \realset^{d \times d}$ is its covariance, represented by a \emph{positive semi-definite (PSD)} matrix.

The posterior distribution of $\theta$ at round $t$ is $\cN(\cdot; \hat{\theta}_t, \hat{\Sigma}_t)$, where
\begin{align*}
  \hat{\theta}_t
  = \hat{\Sigma}_t \left(\Sigma_0^{-1} \theta_0 +
  \sigma^{-2} \sum_{\ell = 1}^{t - 1} A_\ell Y_\ell\right)\,, \quad
  \hat{\Sigma}_t
  = (\Sigma_0^{-1} + G_t)^{-1}\,, \quad
  G_t
  = \sigma^{-2} \sum_{\ell = 1}^{t - 1} A_\ell A_\ell\T\,.
\end{align*} 
Here $\hat{\theta}_t$ and $\hat{\Sigma}_t$ are the posterior mean and covariance of $\theta$, respectively, and $G_t$ is the outer product of the feature vectors of the taken actions up to round $t$. These formulas follow from a classic result, that the posterior distribution of a linear model parameter with a Gaussian prior and observations is a Gaussian \citep{bishop06pattern}. The Bayesian UCB of action $a$ at round $t$ is $U_{t, a} = a\T \hat{\theta}_t + C_{t, a}$, where $C_{t, a} = \sqrt{2 \log(1 / \delta)} \normw{a}{\hat{\Sigma}_t}$ is the confidence interval width, $\delta \in (0, 1)$ is a failure probability of the confidence interval, and $\normw{a}{M} = \sqrt{a\T M a}$.

%% file: RegretBounds.tex
\section{Logarithmic Bayes Regret Upper Bounds}
\label{sec:regret bounds}

In this section, we present finite-time logarithmic Bayes regret bounds for \bayesucb. We derive them for both $K$-armed and linear bandits. One bound matches an existing lower bound of \citet{lai87adaptive} asymptotically and all improve upon prior $\tilde{O}(\sqrt{n})$ bounds. We discuss this in detail in \cref{sec:prior work}.

\subsection{\bayesucb in Gaussian Bandit}
\label{sec:regret bound bayesucb}

Our first regret bound is for \bayesucb in a $K$-armed Gaussian bandit. It depends on random gaps. To control the gaps, we clip them as $\Delta_a^\varepsilon = \max \set{\Delta_a, \varepsilon}$. The bound is stated below.

\begin{theorem}
\label{thm:gap-dependent bayesucb} For any $\varepsilon > 0$ and $\delta \in (0, 1)$, the $n$-round Bayes regret of \bayesucb in a $K$-armed Gaussian bandit is bounded as
\begin{align*}
  R(n)
  \leq \E{\sum_{a \neq A_*}
  \frac{8 \sigma^2 \log(1 / \delta)}{\Delta_a^\varepsilon} -
  \frac{\sigma^2 \Delta_a^\varepsilon}{\sigma_0^2}} + C\,,
\end{align*}
where $C = \varepsilon n + 2 (\sqrt{2 \log(1 / \delta)} + 2 K) \sigma_0 K n \delta$ is a low-order term.
\end{theorem}

The proof is in \cref{sec:gap-dependent bayesucb proof}. For $\varepsilon = 1 / n$ and $\delta = 1 / n$, the bound is $O(c_\Delta \log n)$, where $c_\Delta$ is a constant depending on the gaps of random bandit instances. The dependence on $\sigma_0$ in the low-order term $C$ can be reduced to $\min \set{\sigma_0, \sigma}$ by a more elaborate analysis, where the regret of taking each action for the first time is bounded separately. This also applies to \cref{thm:prior-dependent bayesucb}.

Now we derive an upper bound on \cref{thm:gap-dependent bayesucb} that eliminates the dependence on random gaps. To state it, we need to introduce additional notation. For any action $a$, we denote all action parameters except for $a$ by $\theta_{- a} = (\theta_1, \dots, \theta_{a - 1}, \theta_{a + 1}, \dots, \theta_K)$ and the corresponding optimal action in $\theta_{- a}$ by $\theta^*_a = \max_{j \in \cA \setminus \set{a}} \theta_j$. We denote by $h_a$ the prior density of $\theta_a$ and by $h_{- a}$ the prior density of $\theta_{- a}$. Since the prior is factored (\cref{sec:gaussian bandit}), note that $h(\theta) = h_a(\theta_a) h_{- a}(\theta_{- a})$ for any $\theta$ and action $a$. To keep the result clean, we state it for a \say{sufficiently} large prior variance. A complete statement for all prior variances is given in \cref{sec:complete prior-dependent regret bound}. We note that the setting of small prior variances favors Bayesian algorithms since their regret decreases with a more informative prior. In fact, we show in \cref{sec:complete prior-dependent regret bound} that the regret of \bayesucb is $O(1)$ for a sufficiently small $\sigma_0$.

\begin{corollary}
\label{thm:prior-dependent bayesucb} Let $\sigma_0^2 \geq \frac{1}{8 \log(1 / \delta) \, n^2 \log \log n}$. Then there exist functions $\xi_a: \realset \to \left[\frac{1}{n}, \frac{1}{\sqrt{\log n}}\right]$ such that the $n$-round Bayes regret of \bayesucb in a $K$-armed Gaussian bandit is bounded as
\begin{align*}
  R(n)
  \leq \left[8 \sigma^2 \log(1 / \delta) \log n -
  \frac{\sigma^2}{2 \sigma_0^2 \log n}\right] \sum_{a \in \cA} \int_{\theta_{- a}}
  h_a(\theta_a^* - \xi_a(\theta_a^*)) \, h_{- a}(\theta_{- a}) \dif \theta_{- a} + C\,,
\end{align*}
where $C = 8 \sigma^2 K \log(1 / \delta) \sqrt{\log n} + 2 (\sqrt{2 \log(1 / \delta)} + 2 K) \sigma_0 K n \delta + 1$ is a low-order term.
\end{corollary}

The proof is in \cref{sec:prior-dependent bayesucb proof}. For $\delta = 1 / n$, the bound is $O(c_h \log ^2 n)$, where $c_h$ depends on prior $h$ but not on the gaps of random bandit instances. This bound is motivated by \citet{lai87adaptive}. The terms $\xi_a$ arise due to the intermediate value theorem for function $h_a$. Similar terms appear in \citet{lai87adaptive} but vanish in their final asymptotic claims. The rate $1 / \sqrt{\log n}$ in the definition of $\xi_a$ cannot be reduced to $1 / \text{poly}{\log n}$ without increasing dependence on $n$ in other parts of the bound.

The complexity term $\sum_{a \in \cA} \int_{\theta_{- a}} h_a(\theta_a^* - \xi_a(\theta_a^*)) \, h_{- a}(\theta_{- a}) \dif \theta_{- a}$ in \cref{thm:prior-dependent bayesucb} is the same as in \citet{lai87adaptive} and can be interpreted as follows. Consider the asymptotic regime of $n \to \infty$. Then, since the range of $\xi_a$ is $\left[\frac{1}{n}, \frac{1}{\sqrt{\log n}}\right]$, the term simplifies to $\sum_{a \in \cA} \int_{\theta_{- a}} h_a(\theta_a^*) h_{- a}(\theta_{- a}) \dif \theta_{- a}$ and can be viewed as the distance between prior means. In a Gaussian bandit with $K = 2$ actions, it has a closed form of $\frac{1}{\sqrt{\pi \sigma_0^2}} \exp\left[- \frac{(\mu_{0, 1} - \mu_{0, 2})^2}{4 \sigma_0^2}\right]$. A general upper bound for $K > 2$ actions is given below.

\begin{lemma}
\label{lem:interpretability} In a $K$-armed Gaussian bandit with prior $h(\cdot)= \cN(\cdot; \mu_0, \sigma_0^2 I_K)$, we have
\begin{align*}
  \sum_{a \in \cA} \int_{\theta_{- a}}
  h_a(\theta_a^*) \, h_{- a}(\theta_{- a}) \dif \theta_{- a}
  \leq \frac{1}{2 \sqrt{\pi \sigma_0^2}} \sum_{a \in \cA} \sum_{a' \neq a}
  \exp\left[- \frac{(\mu_{0, a} - \mu_{0, a'})^2}{4 \sigma_0^2}\right]\,.
\end{align*}
\end{lemma}

The bound is proved in \cref{sec:interpretability proof} and has several interesting properties that capture low-regret regimes. First, as the prior becomes more informative and concentrated, $\sigma_0 \to 0$, the bound goes to zero. Second, when the gaps of bandit instances sampled from the prior are large, low regret is also expected. This can happen when the prior means become more separated, $|\mu_{0, a} - \mu_{0, a'}| \to \infty$, or the prior becomes wider, $\sigma_0 \to \infty$. Our bound goes to zero in both of these cases. This also implies that Bayes regret bounds are not necessarily monotone in prior parameters, such as $\sigma_0$.

\subsection{\ucb in Gaussian Bandit}
\label{sec:regret bound ucb1}

Using a similar approach, we prove a Bayes regret bound for \ucb \citep{auer02finitetime}. We view it as \bayesucb (\cref{sec:gaussian bandit}) where $\sigma_0 = \infty$ and each action $a \in \cA$ is initially taken once at round $t = a$. This generalizes classic \ucb to $\sigma^2$-sub-Gaussian noise. An asymptotic Bayes regret bound for \ucb was proved by \citet{lai87adaptive} (claim (i) in their Theorem 3). We derive a finite-time prior-dependent Bayes regret bound below.

\begin{theorem}
\label{thm:ucb1} There exist functions $\xi_a: \realset \to \left[\frac{1}{n}, \frac{1}{\sqrt{\log n}}\right]$ such that the $n$-round Bayes regret of \ucb in a $K$-armed Gaussian bandit is bounded as
\begin{align*}
  R(n)
  \leq 8 \sigma^2 \log(1 / \delta) \log n \sum_{a \in \cA} \int_{\theta_{- a}}
  h_a(\theta_a^* - \xi_a(\theta_a^*)) \, h_{- a}(\theta_{- a}) \dif \theta_{- a} + C\,,
\end{align*}
where $C = 8 \sigma^2 K \log(1 / \delta) \sqrt{\log n} + 2 (\sqrt{2 \log(1 / \delta)} + 2 K) \sigma K n \delta + \sum_{a \in \cA} \E{\Delta_a} + 1$.
\end{theorem}

The proof is in \cref{sec:ucb1 proof}. For $\delta = 1 / n$, the bound is $O(c_h \log ^2 n)$ and similar to \cref{thm:prior-dependent bayesucb}. The main difference is in the additional factor $\frac{\sigma^2}{2 \sigma_0^2 \log n}$ in \cref{thm:prior-dependent bayesucb}, which decreases the bound. This means that the regret of \bayesucb improves as $\sigma_0$ decreases while that of \ucb may not change much. In fact, the regret bound of \bayesucb is $O(1)$ as $\sigma_0 \to 0$ (\cref{sec:complete prior-dependent regret bound}) while that of \ucb remains logarithmic. This is expected because \bayesucb has more information about the random instance $\theta$ as $\sigma_0$ decreases, while the frequentist algorithm is oblivious to the prior.

\subsection{\bayesucb in Bernoulli Bandit}
\label{sec:regret bound bernoulli bayesucb}

\cref{thm:gap-dependent bayesucb,thm:prior-dependent bayesucb} can be straightforwardly extended to Bernoulli bandits because
\begin{align*}
  \condprob{|\theta_a - \hat{\theta}_{t, a}| \geq C_{t, a}}{H_t}
  \leq 2 \delta
\end{align*}
holds for any action $a$ and history $H_t$ (\cref{sec:bernoulli bandit}). We state the extension below and prove it in \cref{sec:bernoulli bayesucb proof}.

\begin{theorem}
\label{thm:bernoulli bayesucb} For any $\varepsilon > 0$ and $\delta \in (0, 1)$, the $n$-round Bayes regret of \bayesucb in a $K$-armed Bernoulli bandit is bounded as 
\begin{align*}
  R(n)
  \leq \E{\sum_{a \neq A_*}
  \frac{2 \log(1 / \delta)}{\Delta_a^\varepsilon} -
  (\alpha_a + \beta_a + 1) \Delta_a^\varepsilon} + C\,,
\end{align*}
where $C = \varepsilon n + 2 K n \delta$ is a low-order term.

Moreover, let $\lambda = \min_{a \in \cA} \alpha_a + \beta_a + 1$ and $\lambda \leq 2 \log(1 / \delta) \, n^2 \log \log n$. Then
\begin{align*}
  R(n)
  \leq \left[2 \log(1 / \delta) \log n -
  \frac{\lambda}{2 \log n}\right] \sum_{a \in \cA} \int_{\theta_{- a}}
  h_a(\theta_a^* - \xi_a(\theta_a^*)) \, h_{- a}(\theta_{- a}) \dif \theta_{- a} + C\,,
\end{align*}
where $C = 2 K \log(1 / \delta) \sqrt{\log n} + 2 K n \delta + 1$ is a low-order term.
\end{theorem}

\subsection{\bayesucb in Linear Bandit}
\label{sec:regret bound linear bayesucb}

Now we present a gap-dependent Bayes regret bound for \bayesucb in a linear bandit with a finite number of actions. The bound depends on a random minimum gap. To control the gap, we clip it as $\Delta_{\min}^\varepsilon = \max \set{\Delta_{\min}, \varepsilon}$.

\begin{theorem}
\label{thm:linear bayesucb} Suppose that $\normw{\theta}{2} \leq L_*$ holds with probability at least $1 - \delta_*$. Let $\normw{a}{2} \leq L$ for any action $a \in \cA$. Then for any $\varepsilon > 0$ and $\delta \in (0, 1)$, the $n$-round Bayes regret of linear \bayesucb is bounded as
\begin{align*}
  R(n)
  \leq 8 \E{\frac{1}{\Delta_{\min}^\varepsilon}}
  \frac{\sigma_{0, \max}^2 d}{\log\left(1 + \frac{\sigma_{0, \max}^2}{\sigma^2}\right)}
  \log\left(1 + \frac{\sigma_{0, \max}^2 n}{\sigma^2 d}\right) \log(1 / \delta) +
  \varepsilon n + 4 L L_* K n \delta
\end{align*}
with probability at least $1 - \delta_*$, where $\sigma_{0, \max} = \sqrt{\lambda_1(\Sigma_0)} L$.
\end{theorem}

The proof is in \cref{sec:linear bayesucb proof}. For $\varepsilon = 1 / n$ and $\delta = 1 / n$, the bound is $O(d \, c_\Delta \log^2 n)$, where $c_\Delta$ is a constant depending on the gaps of random bandit instances. The bound is remarkably similar to the frequentist $O(d \, \Delta_{\min}^{-1} \log^2 n)$ bound in Theorem 5 of \citet{abbasi-yadkori11improved}, where $\Delta_{\min}$ is the minimum gap. There are two differences. First, we integrate $\Delta_{\min}^{-1}$ over the prior. Second, our bound decreases as the prior becomes more informative, $\sigma_{0, \max} \to 0$.

In a Gaussian bandit, the bound becomes $O(K \E{1 / \Delta_{\min}^\varepsilon} \log^2 n)$. Therefore, it is comparable to \cref{thm:gap-dependent bayesucb} up to an additional logarithmic factor in $n$. This is due to a more general proof technique, which allows for dependencies between the mean rewards of actions. We also note that \cref{thm:linear bayesucb} does not assume that the prior is factored, unlike \cref{thm:gap-dependent bayesucb}.

%% file: PriorWork.tex
\section{Comparison to Prior Works}
\label{sec:prior work}

This section is organized as follows. In \cref{sec:matching lower bound}, we show that the bound in \cref{thm:bernoulli bayesucb} matches an existing lower bound of \citet{lai87adaptive} asymptotically. In \cref{sec:prior bayes regret upper bounds}, we compare our logarithmic bounds to prior $\tilde{O}(\sqrt{n})$ bounds. Finally, in \cref{sec:technical novelty}, we outline the key steps in our analyses and how they differ from prior works.

\subsection{Matching Lower Bound}
\label{sec:matching lower bound}

In frequentist bandit analyses, it is standard to compare asymptotic lower bounds to finite-time upper bounds because finite-time logarithmic lower bounds do not exist \citep{lattimore19bandit}. We follow the same approach when arguing that our finite-time upper bounds are order optimal.

The results in \citet{lai87adaptive} are for single-parameter exponential-family reward distributions, which excludes Gaussian rewards. Therefore, we argue about the tightness of \cref{thm:bernoulli bayesucb} only. Specifically, we take the second bound in \cref{thm:bernoulli bayesucb}, set $\delta = 1 / n$, and let $n \to \infty$. In this case, $\frac{\lambda}{2 \log n} \to 0$ and $\xi_a(\cdot) \to 0$, and the bound matches up to constant factors the lower bound in \citet{lai87adaptive} (claim (ii) in their Theorem 3), which is
\begin{align}
  \Omega\left(\log^2 n \sum_{a \in \cA} \int_{\theta_{- a}}
  h_a(\theta_a^*) \, h_{- a}(\theta_{- a}) \dif \theta_{- a}\right)\,.
  \label{eq:lower bound}
\end{align}
\citet{lai87adaptive} also matched this lower bound with an asymptotic upper bound for a frequentist policy.

Our finite-time upper bounds also reveal an interesting difference from the asymptotic lower bound in \eqref{eq:lower bound}, which may deserve more future attention. More specifically, the regret bound of \bayesucb (\cref{thm:prior-dependent bayesucb}) improves with prior information while that of \ucb (\cref{thm:ucb1}) does not. We observe these improvements empirically as well (\cref{sec:experiments,sec:bakeoff}). However, both bounds are the same asymptotically. This is because the benefit of knowing the prior vanishes in asymptotic analyses, since $\frac{\sigma^2}{2 \sigma_0^2 \log n} \to 0$ in \cref{thm:prior-dependent bayesucb} as $n \to \infty$. This motivates the need for finite-time logarithmic Bayes regret lower bounds, which do not exist.

\subsection{Prior Bayes Regret Upper Bounds}
\label{sec:prior bayes regret upper bounds}

\cref{thm:gap-dependent bayesucb,thm:prior-dependent bayesucb} are major improvements upon existing $\tilde{O}(\sqrt{n})$ bounds. For instance, take a prior-dependent bound in Lemma 4 of \citet{kveton21metathompson}, which holds for both \bayesucb and Thompson sampling due to a well-known equivalence of their analyses \citep{russo14learning,hong20latent}. For $\delta = 1 / n$, their leading term becomes
\begin{align}
  4 \sqrt{2 \sigma^2 K \log n}
  \left(\sqrt{n + \sigma^2 \sigma_0^{-2} K} -
  \sqrt{\sigma^2 \sigma_0^{-2} K}\right)\,.
  \label{eq:sqrt bound}
\end{align}
Similarly to \cref{thm:gap-dependent bayesucb,thm:prior-dependent bayesucb}, \eqref{eq:sqrt bound} decreases as the prior concentrates and becomes more informative, $\sigma_0 \to 0$. However, the bound is $\tilde{O}(\sqrt{n})$. Moreover, it does not depend on prior means $\mu_0$ or the gaps of random bandit instances. Therefore, it cannot capture low-regret regimes due to large random gaps $\Delta_a^\varepsilon$ in \cref{thm:gap-dependent bayesucb} or a small complexity term in \cref{thm:prior-dependent bayesucb}. We demonstrate it empirically in \cref{sec:experiments}.

When the random gaps $\Delta_a^\varepsilon$ in \cref{thm:gap-dependent bayesucb} are small or the complexity term in \cref{thm:prior-dependent bayesucb} is large, our bounds can be worse than $\tilde{O}(\sqrt{n})$ bounds. This is analogous to the relation of the gap-dependent and gap-free frequentist bounds \citep{lattimore19bandit}. Specifically, a gap-dependent bound of \ucb in a $K$-armed bandit with $1$-sub-Gaussian rewards (Theorem 7.1) is $O(K \Delta_{\min}^{-1} \log n)$, where $\Delta_{\min}$ is the minimum gap. A corresponding gap-free bound (Theorem 7.2) is $O(\sqrt{K n \log n})$. The latter is smaller when the gap is small, $\Delta_{\min} = o(\sqrt{(K \log n) / n})$. To get the best bound, the minimum of the two should be taken, and the same is true in our Bayesian setting.

No prior-dependent Bayes regret lower bound exists in linear bandits. Thus we treat $\E{1 / \Delta_{\min}^\varepsilon}$ in \cref{thm:linear bayesucb} as the complexity term and do not further bound it as in \cref{thm:prior-dependent bayesucb}. To compare our bound fairly to existing $\tilde{O}(\sqrt{n})$ bounds, we derive an $\tilde{O}(\sqrt{n})$ bound in \cref{sec:gap-free bound}, by a relatively minor change in the proof of \cref{thm:linear bayesucb}. A similar bound can be obtained by adapting the proofs of \citet{lu19informationtheoretic} and \citet{hong22hierarchical} to a linear bandit with a finite number of actions. The leading term of the bound is
\begin{align}
  2 \sqrt{\frac{2 \sigma_{0, \max}^2 d n}
  {\log\left(1 + \frac{\sigma_{0, \max}^2}{\sigma^2}\right)}
  \log\left(1 + \frac{\sigma_{0, \max}^2 n}{\sigma^2 d}\right) \log(1 / \delta)}\,.
  \label{eq:linear sqrt bound}
\end{align}
Similarly to \cref{thm:linear bayesucb}, \eqref{eq:linear sqrt bound} decreases as the prior becomes more informative, $\sigma_{0, \max} \to 0$. However, the bound is $\tilde{O}(\sqrt{n})$ and does not depend on the gaps of random bandit instances. Hence it cannot capture low-regret regimes due to a large random minimum gap $\Delta_{\min}$ in \cref{thm:linear bayesucb}. We validate it empirically in \cref{sec:experiments}.

\subsection{Technical Novelty}
\label{sec:technical novelty}

All modern Bayesian analyses follow \citet{russo14learning}, who derived the first finite-time $\tilde{O}(\sqrt{n})$ Bayes regret bounds for \bayesucb and Thompson sampling. The key idea in their analyses is that conditioned on history, the optimal and taken actions are identically distributed, and that the upper confidence bounds are deterministic functions of the history. This is where the randomness of instances in Bayesian bandits is used. Using this, the regret at round $t$ is bounded by the confidence interval width of the taken action, and the usual $\tilde{O}(\sqrt{n})$ bounds can be obtained by summing up the confidence interval widths over $n$ rounds.

The main difference in our work is that we first bound the regret in a fixed bandit instance, similarly to frequentist analyses. The bound involves $\Delta_a^{-1}$ and is derived using biased Bayesian confidence intervals. The rest of our analysis is Bayesian in two parts: we prove that the confidence intervals fail with a low probability and bound random $\Delta_a^{-1}$, following a similar technique to \citet{lai87adaptive}. The resulting logarithmic Bayes regret bounds cannot be derived using the techniques of \citet{russo14learning}, as these become loose when the confidence interval widths are introduced.

Asymptotic logarithmic Bayes regret bounds were derived in \citet{lai87adaptive}. From this analysis, we use only the technique for bounding $\Delta_a^{-1}$ when proving \cref{thm:prior-dependent bayesucb,thm:bernoulli bayesucb}. The central part of our proof is a finite-time per-instance bound on the number of times that a suboptimal action is taken. This quantity is bounded based on the assumption that the action with the highest UCB is taken. A comparable argument in Theorem 2 of \citet{lai87adaptive} is asymptotic and on average over random bandit instances.

%% file: Experiments.tex
\section{Experiments}
\label{sec:experiments}

We experiment with UCB algorithms in two environments: Gaussian bandits (\cref{sec:gaussian bandit}) and linear bandits with Gaussian rewards (\cref{sec:linear bandit}). In both experiments, the horizon is $n = 1\,000$ rounds. All results are averaged over $10\,000$ random runs. Shaded regions in the plots are standard errors of the estimates. They are generally small because the number of runs is high.

\subsection{Gaussian Bandit}
\label{sec:mab experiments}

\begin{figure}[t]
  \centering
  \includegraphics[width=5.5in]{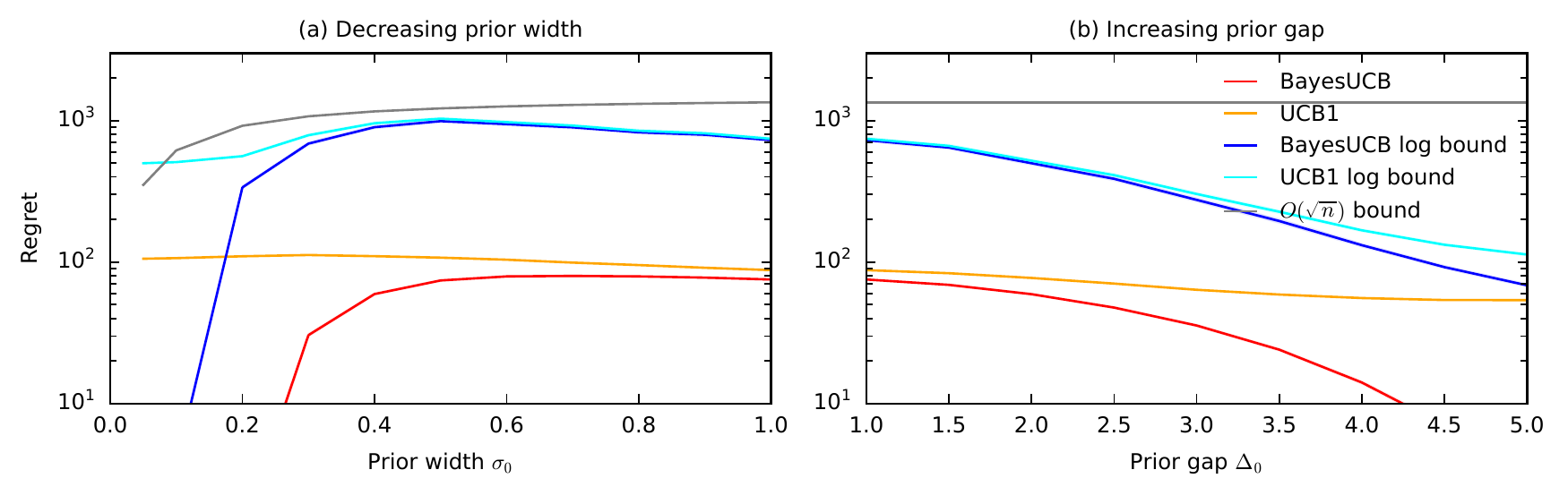}
  \vspace{-0.3in}
  \caption{Gaussian bandit as (a) the prior width $\sigma_0$ and (b) the prior gap $\Delta_0$ change.}
  \label{fig:mab experiments}
\end{figure}

The first problem is a $K$-armed bandit with $K = 10$ actions (\cref{sec:gaussian bandit}). The \emph{prior width} is $\sigma_0 = 1$. The prior mean is $\mu_0$, where $\mu_{0, 1} = \Delta_0$ and $\mu_{0, a} = 0$ for $a > 1$. We set $\Delta_0 = 1$ and call it the \emph{prior gap}. We vary $\sigma_0$ and $\Delta_0$ in our experiments, and observe how the regret and its upper bounds change as the problem hardness varies (\cref{sec:regret bound bayesucb,sec:prior bayes regret upper bounds}).

We plot five trends: (a) Bayes regret of \bayesucb. (b) Bayes regret of \ucb (\cref{sec:regret bound ucb1}), which is a comparable frequentist algorithm to \bayesucb. (c) A leading term of the \bayesucb regret bound in \cref{thm:gap-dependent bayesucb}, where $\varepsilon = 1 / n$ and $\delta = 1 / n$. (d) A leading term of the \ucb regret bound: This is the same as (c) with $\sigma_0 = \infty$. (e) An existing $\tilde{O}(\sqrt{n})$ regret bound in \eqref{eq:sqrt bound}.

Our results are reported in \cref{fig:mab experiments}. We observe three major trends. First, the regret of \bayesucb decreases as the problem becomes easier, either $\sigma_0 \to 0$ or $\Delta_0 \to \infty$. It is also lower than that of \ucb, which does not leverage the prior. Second, the regret bound of \bayesucb is tighter than that of \ucb, due to capturing the benefit of the prior. Finally, the logarithmic regret bounds are tighter than the $\tilde{O}(\sqrt{n})$ bound. In addition, the $\tilde{O}(\sqrt{n})$ bound depends on the prior only through $\sigma_0$ and thus remains constant as the prior gap $\Delta_0$ changes.

In \cref{sec:bakeoff}, we compare \bayesucb to \ucb more comprehensively for various $K$, $\sigma$, $\Delta_0$, and $\sigma_0$. In all experiments, \bayesucb has a lower regret than \ucb. This also happens when the noise is not Gaussian, which a testament to the robustness of Bayesian methods to model misspecification.

\subsection{Linear Bandit with Gaussian Rewards}
\label{sec:linear bandit experiments}

\begin{figure}[t]
  \centering
  \includegraphics[width=5.5in]{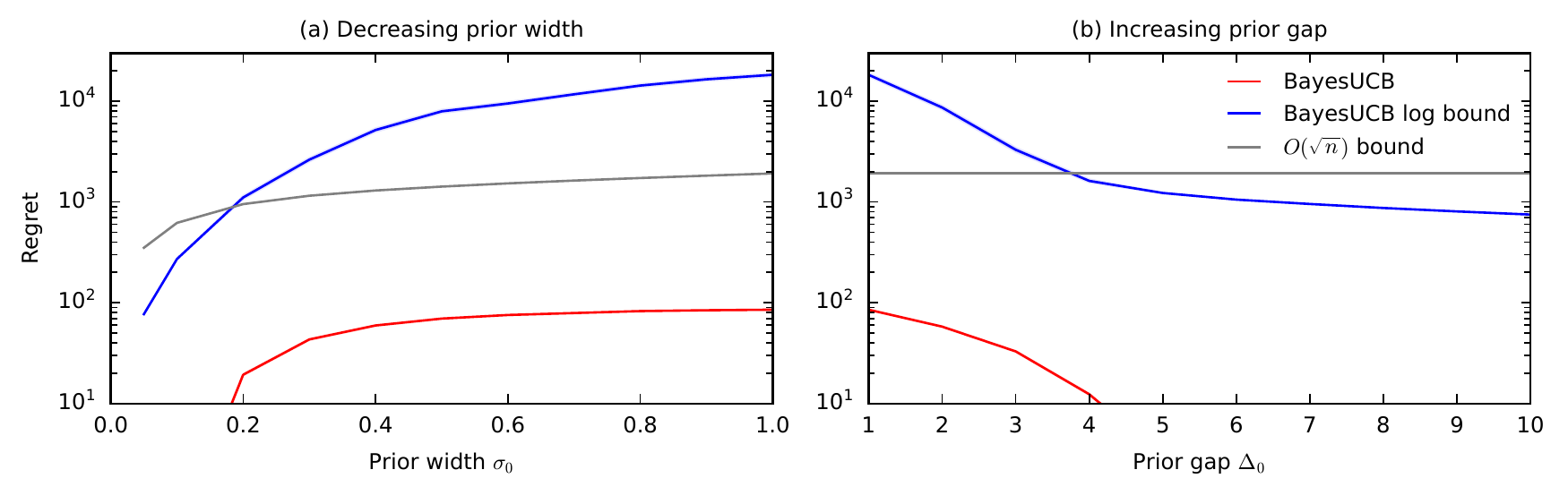}
  \vspace{-0.3in}
  \caption{Linear bandit as (a) the prior width $\sigma_0$ and (b) the prior gap $\Delta_0$ change.}
  \label{fig:linear experiments}
\end{figure}

The second problem is a linear bandit with $K = 30$ actions in $d = 10$ dimensions (\cref{sec:linear bandit}). The prior covariance is $\Sigma_0 = \sigma_0^2 I_d$. The prior mean is $\theta_0$, where $\theta_{0, 1} = \Delta_0$ and $\theta_{0, i} = -1$ for $i > 1$. As in \cref{sec:mab experiments}, we set $\Delta_0 = 1$ and call it the \emph{prior gap}. The action set $\cA$ is generated as follows. The first $d$ actions are the canonical basis in $\realset^d$. The remaining $K - d$ actions are sampled uniformly at random from the positive orthant and scaled to unit length. This ensures that the first action has the highest mean reward, of $\Delta_0$, under the prior mean $\theta_0$. We vary $\sigma_0$ and $\Delta_0$, and observe how the regret and its upper bounds change as the problem hardness varies (\cref{sec:regret bound linear bayesucb,sec:prior bayes regret upper bounds}). We plot three trends: (a) Bayes regret of \bayesucb. (b) A leading term of the \bayesucb regret bound in \cref{thm:linear bayesucb}, where $\sigma_{0, \max} = \sigma_0$, $\varepsilon = 1 / n$, and $\delta = 1 / n$. (c) An existing $\tilde{O}(\sqrt{n})$ regret bound in \eqref{eq:linear sqrt bound}.

Our results are reported in \cref{fig:linear experiments} and we observe three major trends. First, the regret of \bayesucb decreases as the problem becomes easier, either $\sigma_0 \to 0$ or $\Delta_0 \to \infty$. Second, the regret bound of \bayesucb decreases as the problem becomes easier. Finally, our logarithmic regret bound can also be tighter than the $\tilde{O}(\sqrt{n})$ bound. In particular, the $\tilde{O}(\sqrt{n})$ bound depends on the prior only through $\sigma_0$, and thus remains constant as the prior gap $\Delta_0$ changes. We discuss when our bounds could be looser than $\tilde{O}(\sqrt{n})$ bounds in \cref{sec:prior bayes regret upper bounds}.

%% file: Conclusions.tex
\section{Conclusions}
\label{sec:conclusions}

Finite-time logarithmic frequentist regret bounds are the standard way of analyzing $K$-armed bandits \citep{auer02finitetime,garivier11klucb,agrawal12analysis}. In our work, we prove the first comparable finite-time bounds, logarithmic in $n$, in the Bayesian setting. This is a major step in theory and a significant improvement upon prior $\tilde O(\sqrt{n})$ Bayes regret bounds that have become standard in the literature. Comparing to frequentist regret bounds, our bounds capture the value of prior information given to the learner. Our proof technique is general and we also apply it to linear bandits.

This work can be extended in many directions. First, our analyses only need closed-form posteriors, which are available for other reward distributions than Gaussian and Bernoulli. Second, our linear bandit analysis (\cref{sec:regret bound linear bayesucb}) seems preliminary when compared to our multi-armed bandit analyses. As an example, the complexity term $\E{1 / \Delta_{\min}^\varepsilon}$ in \cref{thm:linear bayesucb} could be bounded as in \cref{thm:prior-dependent bayesucb,thm:bernoulli bayesucb}. We do not do this because the main reason for deriving the $O(c_h \log^2 n)$ bound in \cref{thm:bernoulli bayesucb}, an upper bound on the corresponding $O(c_\Delta \log n)$ bound, is that it matches the lower bound in \eqref{eq:lower bound}. No such instance-dependent lower bound exists in Bayesian linear bandits. Third, we believe that our approach can be extended to information-theory bounds \citep{russo16information} and discuss it in \cref{sec:information-theory note}. Fourth, although we only analyze \bayesucb, we believe that similar guarantees can be obtained for Thompson sampling. Finally, we would like to extend our results to reinforcement learning, for instance by building on the work of \citet{lu19informationtheoretic}.

There have been recent attempts in theory \citep{wagenmaker23instanceoptimality} to design general adaptive algorithms with finite-time instance-dependent bounds based on optimal allocations. The promise of these methods is a higher statistical efficiency than exploring by optimism, which we adopt in this work. One of their shortcomings is that they are not guaranteed to be computationally efficient, as discussed in Section 2.2 of \citet{wagenmaker23instanceoptimality}. This work is also frequentist.

%% file: Proofs.tex
\section{Proofs}
\label{sec:proofs}

We present a general approach for deriving finite-time logarithmic Bayes regret bounds. We start with $K$-armed bandits and then extend it to linear bandits.

\subsection{Proof of \cref{thm:gap-dependent bayesucb}}
\label{sec:gap-dependent bayesucb proof}

Let $E_t = \set{\forall a \in \cA: |\theta_a - \hat{\theta}_{t, a}| \leq C_{t, a}}$ be the event that all confidence intervals at round $t$ hold. Fix $\varepsilon > 0$. We start with decomposing the $n$-round regret as
\begin{align}
  \sum_{t = 1}^n \E{\Delta_{A_t}}
  \leq {} & \sum_{t = 1}^n \E{\Delta_{A_t} \I{\Delta_{A_t} \geq \varepsilon, E_t}} +
  \sum_{t = 1}^n \E{\Delta_{A_t} \I{\Delta_{A_t} < \varepsilon}} + {}
  \label{eq:regret decomposition} \\
  & \sum_{t = 1}^n \E{\Delta_{A_t} \I{\bar{E}_t}}\,.
  \nonumber
\end{align} 
We bound the first term using the design of \bayesucb and its closed-form posteriors.

\textbf{Case 1: Event $E_t$ occurs and the gap is large, $\Delta_{A_t} \geq \varepsilon$.} Since $E_t$ occurs and the action with the highest UCB is taken, the regret at round $t$ can be bounded as
\begin{align*}
  \Delta_{A_t}
  = \theta_{A_*} - \theta_{A_t}
  \leq \theta_{A_*} - U_{t, A_*} + U_{t, A_t} - \theta_{A_t}
  \leq U_{t, A_t} - \theta_{A_t}
  \leq 2 C_{t, A_t}\,.
\end{align*}
In the second inequality, we use that $\theta_{A_*} \leq U_{t, A_*}$ on event $E_t$. This implies that on event $E_t$, action $a$ can be taken only if
\begin{align*}
  \Delta_a
  \leq 2 C_{t, a}
  = 2 \sqrt{2 \hat{\sigma}_{t, a}^2 \log(1 / \delta)}
  = 2 \sqrt{\frac{2 \log(1 / \delta)}{\sigma_0^{-2} + \sigma^{-2} N_{t, a}}}\,,
\end{align*}
which can be rearranged as
\begin{align*}
  N_{t, a}
  \leq \frac{8 \sigma^2 \log(1 / \delta)}{\Delta_a^2} - \sigma^2 \sigma_0^{-2}\,.
\end{align*}
Therefore, the number of times that action $a$ is taken in $n$ rounds while all confidence intervals hold, $N_a = \sum_{t = 1}^n \I{A_t = a, E_t}$, is bounded as
\begin{align}
  N_a
  \leq \frac{8 \sigma^2 \log(1 / \delta)}{\Delta_a^2} - \sigma^2 \sigma_0^{-2}\,.
  \label{eq:action bound}
\end{align}
Now we apply this inequality to bound the first term in \eqref{eq:regret decomposition} as
\begin{align}
  \sum_{t = 1}^n \E{\Delta_{A_t} \I{\Delta_{A_t} \geq \varepsilon, E_t}}
  \leq \E{\sum_{a \neq A_*}
  \left(\frac{8 \sigma^2 \log(1 / \delta)}{\Delta_a} - \sigma^2 \sigma_0^{-2} \Delta_a\right)
  \I{\Delta_a \geq \varepsilon}}\,.
  \label{eq:gap-dependent regret}
\end{align}

\textbf{Case 2: The gap is small, $\Delta_{A_t} < \varepsilon$.} Then naively $\sum_{t = 1}^n \E{\Delta_{A_t} \I{\Delta_{A_t} < \varepsilon}} < \varepsilon n$.

\textbf{Case 3: Event $E_t$ does not occur.} The last term in \eqref{eq:regret decomposition} can be bounded as
\begin{align}
  & \E{\Delta_{A_t} \I{\bar{E}_t}}
  \label{eq:bad events} \\
  & = \E{\condE{(\theta_{A_*} - \theta_{A_t}) \I{\bar{E}_t}}{H_t}}
  \nonumber \\
  & \leq \E{\condE{(\theta_{A_*} - U_{t, A_*}) \I{\bar{E}_t}}{H_t} +
  \condE{(U_{t, A_t} - \theta_{A_t}) \I{\bar{E}_t}}{H_t}}
  \nonumber \\
  & \leq \E{\condE{(\theta_{A_*} - \hat{\theta}_{t, A_*}) \I{\bar{E}_t}}{H_t} +
  \condE{(\hat{\theta}_{t, A_t} - \theta_{A_t}) \I{\bar{E}_t}}{H_t} +
  \condE{C_{t, A_t} \I{\bar{E}_t}}{H_t}}\,.
  \nonumber
\end{align}
To bound the resulting terms, we use that $\theta_a - \hat{\theta}_{t, a} \mid H_t \sim \cN(0, \hat{\sigma}_{t, a}^2)$.

\begin{lemma}
\label{lem:failure regret} For any action $a \in \cA$, round $t \in [n]$, and history $H_t$,
\begin{align*}
  \condE{(\theta_a - \hat{\theta}_{t, a}) \I{\bar{E}_t}}{H_t}
  \leq 2 \sigma_0 K \delta\,.
\end{align*}
\end{lemma}
\begin{proof}
Let $E_{t, a} = \set{|\theta_a - \hat{\theta}_{t, a}| \leq C_{t, a}}$. We start with decomposing $\bar{E}_t$ into individual $\bar{E}_{t, a}$ as
\begin{align*}
  \condE{(\theta_a - \hat{\theta}_{t, a}) \I{\bar{E}_t}}{H_t}
  \leq \condE{|\theta_a - \hat{\theta}_{t, a}| \I{\bar{E}_{t, a}}}{H_t} +
  \sum_{a' \neq a} \condE{|\theta_a - \hat{\theta}_{t, a}| \I{\bar{E}_{t, a'}}}{H_t}\,.
\end{align*}
To bound the first term, we use that $\theta_a - \hat{\theta}_{t, a} \mid H_t \sim \cN(0, \hat{\sigma}_{t, a}^2)$. Thus
\begin{align*}
  \condE{|\theta_a - \hat{\theta}_{t, a}| \I{\bar{E}_{t, a}}}{H_t}
  & \leq \frac{2}{\sqrt{2 \pi \hat{\sigma}_{t, a}^2}}
  \int_{x = C_{t, a}}^\infty
  x \exp\left[- \frac{x^2}{2 \hat{\sigma}_{t, a}^2}\right] \dif x \\
  & = - \sqrt{\frac{2 \hat{\sigma}_{t, a}^2}{\pi}}
  \int_{x = C_{t, a}}^\infty \frac{\partial}{\partial x}
  \left(\exp\left[- \frac{x^2}{2 \hat{\sigma}_{t, a}^2}\right]\right) \dif x \\
  & = \sqrt{\frac{2 \hat{\sigma}_{t, a}^2}{\pi}} \delta
  \leq \sigma_0 \delta\,.
\end{align*}
To bound the second term, we use the independence of the distributions for $a$ and $a'$,
\begin{align*}
  \condE{|\theta_a - \hat{\theta}_{t, a}| \I{\bar{E}_{t, a'}}}{H_t}
  = \condE{|\theta_a - \hat{\theta}_{t, a}|}{H_t} \condprob{\bar{E}_{t, a'}}{H_t}\,.
\end{align*}
The probability is at most $2 \delta$ and the expectation can be bounded as
\begin{align*}
  \condE{|\theta_a - \hat{\theta}_{t, a}|}{H_t}
  = \condE{\sqrt{(\theta_a - \hat{\theta}_{t, a})^2}}{H_t}
  \leq \sqrt{\condE{(\theta_a - \hat{\theta}_{t, a})^2}{H_t}}
  = \hat{\sigma}_{t, a}
  \leq \sigma_0\,.
\end{align*}
This completes the proof.
\end{proof}

The first two terms in \eqref{eq:bad events} can be bounded using a union bound over $a \in \cA$ and \cref{lem:failure regret}. For the last term, we use that $C_{t, a} \leq \sqrt{2 \sigma_0^2 \log(1 / \delta)}$ and a union bound in $\condprob{\bar{E}_t}{H_t}$ to get
\begin{align*}
  \condE{C_{t, A_t} \I{\bar{E}_t}}{H_t}
  \leq \sqrt{2 \sigma_0^2 \log(1 / \delta)} \condprob{\bar{E}_t}{H_t}
  \leq 2 \sqrt{2 \log(1 / \delta)} \sigma_0 K \delta\,.
\end{align*}
Finally, we sum up the upper bounds on \eqref{eq:bad events} over all rounds $t \in [n]$.

\subsection{Proof of \cref{thm:prior-dependent bayesucb}}
\label{sec:prior-dependent bayesucb proof}

The key idea in the proof is to integrate out the random gap in \eqref{eq:gap-dependent regret}. Fix action $a \in \cA$ and thresholds $\varepsilon_2 > \varepsilon > 0$. We consider two cases.

\textbf{Case 1: Diminishing gaps $\varepsilon < \Delta_a \leq \varepsilon_2$.} Let
\begin{align*}
  \xi_a(\theta_a^*)
  = \argmax_{x \in [\varepsilon, \varepsilon_2]} h_a(\theta_a^* - x)
\end{align*}
and $N_a$ be defined as in \cref{sec:gap-dependent bayesucb proof}. Then
\begin{align*}
  \E{\Delta_a N_a \I{\varepsilon < \Delta_a < \varepsilon_2}}
  & = \int_{\theta_{- a}}
  \int_{\theta_a = \theta_a^* - \varepsilon_2}^{\theta_a^* - \varepsilon}
  \Delta_a N_a h_a(\theta_a) \dif \theta_a \,
  h_{- a}(\theta_{- a}) \dif \theta_{- a} \\
  & \leq \int_{\theta_{- a}}
  \left(\int_{\theta_a = \theta_a^* - \varepsilon_2}^{\theta_a^* - \varepsilon}
  \Delta_a N_a \dif \theta_a\right)
  h_a(\theta_a^* - \xi_a(\theta_a^*)) \, h_{- a}(\theta_{- a}) \dif \theta_{- a}\,,
\end{align*}
where the inequality is by the definition of $\xi_a$. Now the inner integral is independent of $h_a$ and thus can be easily bounded. Specifically, the upper bound in \eqref{eq:action bound} and simple integration yield
\begin{align*}
  \int_{\theta_a = \theta_a^* - \varepsilon_2}^{\theta_a^* - \varepsilon}
  \Delta_a N_a \dif \theta_a
  & \leq \int_{\theta_a = \theta_a^* - \varepsilon_2}^{\theta_a^* - \varepsilon}
  \left(\frac{8 \sigma^2 \log(1 / \delta)}{\theta_a^* - \theta_a} -
  \sigma^2 \sigma_0^{-2} (\theta_a^* - \theta_a)\right) \dif \theta_a \\
  & = 8 \sigma^2 \log(1 / \delta) (\log \varepsilon_2 - \log \varepsilon) -
  \frac{\sigma^2 (\varepsilon_2^2- \varepsilon^2)}{2 \sigma_0^2}\,.
\end{align*}
For $\varepsilon = 1 / n$ and $\varepsilon_2 = 1 / \sqrt{\log n}$, we get
\begin{align}
  \int_{\theta_a = \theta_a^* - \varepsilon_2}^{\theta_a^* - \varepsilon}
  \Delta_a N_a \dif \theta_a
  & \leq 8 \sigma^2 \log(1 / \delta) \log n - \frac{\sigma^2}{2 \sigma_0^2 \log n} +
  \frac{\sigma^2}{2 \sigma_0^2 n^2} - 4 \sigma^2 \log(1 / \delta) \log \log n
  \label{eq:diminishing gaps} \\
  & \leq 8 \sigma^2 \log(1 / \delta) \log n - \frac{\sigma^2}{2 \sigma_0^2 \log n}\,.
  \nonumber
\end{align}
The last inequality holds for $\sigma_0^2 \geq \frac{1}{8 \log(1 / \delta) \, n^2 \log \log n}$.

\textbf{Case 2: Large gaps $\Delta_a > \varepsilon_2$.} Here we use \eqref{eq:action bound} together with $\varepsilon_2 = 1 / \sqrt{\log n}$ to get
\begin{align}
  \E{\Delta_a N_a \I{\Delta_a > \varepsilon_2}}
  \leq \E{\frac{8 \sigma^2 \log(1 / \delta)}{\Delta_a} \I{\Delta_a > \varepsilon_2}}
  < 8 \sigma^2 \log(1 / \delta) \sqrt{\log n}\,.
  \label{eq:large gaps}
\end{align} 
Finally, we chain all inequalities.

\subsection{Proof of \cref{lem:interpretability}}
\label{sec:interpretability proof}

We have that
\begin{align*}
  & \sum_{a \in \cA} \int_{\theta_{- a}}
  h_a(\theta_a^*) \, h_{- a}(\theta_{- a}) \dif \theta_{- a} \\
  & \quad \leq \sum_{a \in \cA} \int_{\theta_{- a}}
  \Bigg(\sum_{a' \neq a} h_a(\theta_{a'})\Bigg)
  \Bigg(\prod_{a' \neq a} h_{a'}(\theta_{a'})\Bigg) \dif \theta_{- a} \\
  & \quad = \sum_{a \in \cA} \sum_{a' \neq a} \int_{\theta_{a'}}
  h_a(\theta_{a'}) \, h_{a'}(\theta_{a'}) \dif \theta_{a'} \\
  & \quad = \frac{1}{2 \pi \sigma_0^2}
  \sum_{a \in \cA} \sum_{a' \neq a} \int_{\theta_{a'}}
  \exp\left[- \frac{(\theta_{a'} - \mu_{0, a})^2}{2 \sigma_0^2} -
  \frac{(\theta_{a'} - \mu_{0, a'})^2}{2 \sigma_0^2}\right] \dif \theta_{a'} \\
  & \quad = \frac{1}{2 \sqrt{\pi \sigma_0^2}} \sum_{a \in \cA} \sum_{a' \neq a}
  \exp\left[- \frac{(\mu_{0, a} - \mu_{0, a'})^2}{4 \sigma_0^2}\right]\,,
\end{align*}
where the last step is by completing the square and integrating out $\theta_{a'}$.

\subsection{Proof of \cref{thm:ucb1}}
\label{sec:ucb1 proof}

The regret bound of \ucb is proved similarly to \cref{thm:gap-dependent bayesucb,thm:prior-dependent bayesucb}. This is because \ucb can be viewed as \bayesucb where $\sigma_0 = \infty$ and each action $a \in \cA$ is initially taken once at round $t = a$. Since $\sigma_0 = \infty$, the confidence interval becomes
\begin{align*}
  C_{t, a}
  = \sqrt{\frac{2 \sigma^2 \log(1 / \delta)}{N_{t, a}}}\,.
\end{align*}
The proof differs in two steps. First, the regret in the first $K$ rounds is bounded by $\sum_{a \in \cA} \E{\Delta_a}$. Second, the concentration argument (Case 3 in \cref{sec:gap-dependent bayesucb proof}) changes because the bandit instance $\theta$ is fixed and the estimated model parameter $\hat{\theta}_t$ is random. We detail it below. 

\textbf{Case 3: Event $E_t$ does not occur.} The last term in \eqref{eq:regret decomposition} can be bounded as
\begin{align}
  & \E{\Delta_{A_t} \I{\bar{E}_t}}
  \label{eq:bad events ucb1} \\
  & = \E{\condE{(\theta_{A_*} - \theta_{A_t}) \I{\bar{E}_t}}{\theta}}
  \nonumber \\
  & \leq \E{\condE{(\theta_{A_*} - U_{t, A_*}) \I{\bar{E}_t}}{\theta} +
  \condE{(U_{t, A_t} - \theta_{A_t}) \I{\bar{E}_t}}{\theta}}
  \nonumber \\
  & \leq \E{\condE{(\theta_{A_*} - \hat{\theta}_{t, A_*}) \I{\bar{E}_t}}{\theta} +
  \condE{(\hat{\theta}_{t, A_t} - \theta_{A_t}) \I{\bar{E}_t}}{\theta} +
  \condE{C_{t, A_t} \I{\bar{E}_t}}{\theta}}\,.
  \nonumber
\end{align}
To bound the resulting terms, we use that $\theta_a - \hat{\theta}_{t, a} \mid \theta \sim \cN(0, \sigma^2 / N_{t, a})$.

\begin{lemma}
\label{lem:failure regret ucb1} For any action $a \in \cA$, round $t > K$, and $N_{t, a} \geq 1$,
\begin{align*}
  \condE{(\theta_a - \hat{\theta}_{t, a}) \I{\bar{E}_t}}{\theta}
  \leq 2 \sigma K \delta\,.
\end{align*}
\end{lemma}
\begin{proof}
Let $E_{t, a} = \set{|\theta_a - \hat{\theta}_{t, a}| \leq C_{t, a}}$. We start with decomposing $\bar{E}_t$ into individual $\bar{E}_{t, a}$ as
\begin{align*}
  \condE{(\theta_a - \hat{\theta}_{t, a}) \I{\bar{E}_t}}{\theta}
  \leq \condE{|\theta_a - \hat{\theta}_{t, a}| \I{\bar{E}_{t, a}}}{\theta} +
  \sum_{a' \neq a} \condE{|\theta_a - \hat{\theta}_{t, a}| \I{\bar{E}_{t, a'}}}{\theta}\,.
\end{align*}
To bound the first term, we use that $\theta_a - \hat{\theta}_{t, a} \mid \theta \sim \cN(0, \sigma^2 / N_{t, a})$. Thus
\begin{align*}
  \condE{|\theta_a - \hat{\theta}_{t, a}| \I{\bar{E}_t}}{\theta}
  & \leq \frac{2}{\sqrt{2 \pi \sigma^2 / N_{t, a}}}
  \int_{x = C_{t, a}}^\infty
  x \exp\left[- \frac{x^2}{2 \sigma^2 / N_{t, a}}\right] \dif x \\
  & = - \sqrt{\frac{2 \sigma^2}{\pi N_{t, a}}}
  \int_{x = C_{t, a}}^\infty \frac{\partial}{\partial x}
  \left(\exp\left[- \frac{x^2}{2 \sigma^2 / N_{t, a}}\right]\right) \dif x \\
  & = \sqrt{\frac{2 \sigma^2}{\pi N_{t, a}}} \delta
  \leq \sigma \delta\,.
\end{align*}
To bound the second term, we use the independence of the distributions for $a$ and $a'$,
\begin{align*}
  \condE{|\theta_a - \hat{\theta}_{t, a}| \I{\bar{E}_{t, a'}}}{\theta}
  = \condE{|\theta_a - \hat{\theta}_{t, a}|}{\theta} \condprob{\bar{E}_{t, a'}}{\theta}\,.
\end{align*}
The probability is at most $2 \delta$ and the expectation can be bounded as
\begin{align*}
  \condE{|\theta_a - \hat{\theta}_{t, a}|}{\theta}
  = \condE{\sqrt{(\theta_a - \hat{\theta}_{t, a})^2}}{\theta}
  \leq \sqrt{\condE{(\theta_a - \hat{\theta}_{t, a})^2}{\theta}}
  = \sqrt{\frac{\sigma^2}{N_{t, a}}}
  \leq \sigma\,.
\end{align*}
This completes the proof.
\end{proof}

The first two terms in \eqref{eq:bad events ucb1} can be bounded using a union bound over $a \in \cA$ and \cref{lem:failure regret ucb1}. For the last term, we use that $C_{t, a} \leq \sqrt{2 \sigma^2 \log(1 / \delta)}$ and a union bound in $\condprob{\bar{E}_t}{\theta}$ to get
\begin{align*}
  \condE{C_{t, A_t} \I{\bar{E}_t}}{\theta}
  \leq \sqrt{2 \sigma^2 \log(1 / \delta)} \condprob{\bar{E}_t}{\theta}
  \leq 2 \sqrt{2 \log(1 / \delta)} \sigma K \delta\,.
\end{align*}
Finally, we sum up the upper bounds on \eqref{eq:bad events ucb1} over all rounds $t \in [n]$.

\subsection{Proof of \cref{thm:bernoulli bayesucb}}
\label{sec:bernoulli bayesucb proof}

Let $E_t = \set{\forall a \in \cA: |\theta_a - \hat{\theta}_{t, a}| \leq C_{t, a}}$ be the event that all confidence intervals at round $t$ hold. Fix $\varepsilon > 0$. We decompose the $n$-round regret as in \eqref{eq:regret decomposition} and then bound each resulting term next.

\textbf{Case 1: Event $E_t$ occurs and the gap is large, $\Delta_{A_t} \geq \varepsilon$.} As in \cref{sec:gap-dependent bayesucb proof},
\begin{align*}
  \Delta_{A_t}
  = \theta_{A_*} - \theta_{A_t}
  \leq \theta_{A_*} - U_{t, A_*} + U_{t, A_t} - \theta_{A_t}
  \leq U_{t, A_t} - \theta_{A_t}
  \leq 2 C_{t, A_t}\,.
\end{align*}
In the second inequality, we use that $\theta_{A_*} \leq U_{t, A_*}$ on event $E_t$. This implies that on event $E_t$, action $a$ can be taken only if
\begin{align*}
  N_{t, a}
  \leq \frac{2 \log(1 / \delta)}{\Delta_a^2} - (\alpha_a + \beta_a + 1)\,.
\end{align*}
Now we apply this inequality to bound the first term in \eqref{eq:regret decomposition} as
\begin{align*}
  \sum_{t = 1}^n \E{\Delta_{A_t} \I{\Delta_{A_t} \geq \varepsilon, E_t}}
  \leq \E{\sum_{a \neq A_*}
  \left(\frac{2 \log(1 / \delta)}{\Delta_a} - (\alpha_a + \beta_a + 1) \Delta_a\right)
  \I{\Delta_a \geq \varepsilon}}\,.
\end{align*}

\textbf{Case 2: The gap is small, $\Delta_{A_t} < \varepsilon$.} Then naively $\sum_{t = 1}^n \E{\Delta_{A_t} \I{\Delta_{A_t} < \varepsilon}} < \varepsilon n$.

\textbf{Case 3: Event $E_t$ does not occur.} Since $\theta_a \in [0, 1]$, the last term in \eqref{eq:regret decomposition} can be bounded as
\begin{align*}
  \E{\Delta_{A_t} \I{\bar{E}_t}}
  \leq \E{\condprob{\bar{E}_t}{H_t}}
  \leq 2 K \delta\,.
\end{align*}
This completes the first part of the proof.

The second claim is proved as in \cref{sec:prior-dependent bayesucb proof} and we only comment on what differs. For $\varepsilon = 1 / n$ and $\varepsilon_2 = 1 / \sqrt{\log n}$, \eqref{eq:diminishing gaps} becomes
\begin{align*}
  \int_{\theta_a = \theta_a^* - \varepsilon_2}^{\theta_a^* - \varepsilon}
  \Delta_a N_a \dif \theta_a
  & \leq \int_{\theta_a = \theta_a^* - \varepsilon_2}^{\theta_a^* - \varepsilon}
  \frac{2 \log(1 / \delta)}{\theta_a^* - \theta_a} -
  \lambda (\theta_a^* - \theta_a) \dif \theta_a \\
  & = 2 \log(1 / \delta) (\log \varepsilon_2 - \log \varepsilon) -
  \frac{\lambda (\varepsilon_2^2 - \varepsilon^2)}{2} \\
  & = 2 \log(1 / \delta) \log n - \frac{\lambda}{2 \log n} +
  \frac{\lambda}{2 n^2} - \log(1 / \delta) \log \log n \\
  & \leq 2 \log(1 / \delta) \log n - \frac{\lambda}{2 \log n}\,.
\end{align*}
The last inequality holds for $\lambda \leq 2 \log(1 / \delta) \, n^2 \log \log n$. Moreover, \eqref{eq:large gaps} becomes
\begin{align*}
  \E{\Delta_a N_a \I{\Delta_a > \varepsilon_2}}
  \leq \E{\frac{2 \log(1 / \delta)}{\Delta_a} \I{\Delta_a > \varepsilon_2}}
  < 2 \log(1 / \delta) \sqrt{\log n}\,.
\end{align*} 
This completes the second part of the proof.

\subsection{Proof of \cref{thm:linear bayesucb}}
\label{sec:linear bayesucb proof}

Let
\begin{align}
  E_t
  = \set{\forall a \in \cA: |a\T (\theta - \hat{\theta}_t)|
  \leq \sqrt{2 \log(1 / \delta)} \normw{a}{\hat{\Sigma}_t}}
  \label{eq:linear confidence interval}
\end{align}
be an event that high-probability confidence intervals for mean rewards at round $t$ hold. Our proof has three parts.

\textbf{Case 1: Event $E_t$ occurs and the gap is large, $\Delta_{A_t} \geq \varepsilon$.} Then
\begin{align*}
  \Delta_{A_t}
  & = \frac{1}{\Delta_{A_t}} \Delta_{A_t}^2
  \leq \frac{1}{\Delta_{\min}^\varepsilon}
  (A_*\T \theta - A_t\T \theta)^2
  \leq \frac{1}{\Delta_{\min}^\varepsilon}
  (A_*\T \theta - U_{t, A_*} + U_{t, A_t} - A_t\T \theta)^2 \\
  & \leq \frac{1}{\Delta_{\min}^\varepsilon} (U_{t, A_t} - A_t\T \theta)^2
  \leq \frac{4}{\Delta_{\min}^\varepsilon} C_{t, A_t}^2
  = \frac{8 \log(1 / \delta)}{\Delta_{\min}^\varepsilon}
  \normw{A_t}{\hat{\Sigma}_t}^2\,.
\end{align*}
The first inequality follows from definitions of $\Delta_{A_t}$ and $\Delta_{\min}^\varepsilon$; and that the gap is large, $\Delta_{A_t} \geq \varepsilon$. The second inequality holds because $A_*\T \theta - A_t\T \theta \geq 0$ by definition and $U_{t, A_t} - U_{t, A_*} \geq 0$ by the design of \bayesucb. The third inequality holds because $A_*\T \theta - U_{t, A_*} \leq 0$ on event $E_t$. Specifically, for any action $a \in \cA$ on event $E_t$,
\begin{align*}
  a\T \theta - U_{t, a}
  = a\T (\theta - \hat{\theta}_t) - C_{t, a}
  \leq C_{t, a} - C_{t, a}
  = 0\,.
\end{align*}
The last inequality follows from the definition of event $E_t$. Specifically, for any action $a \in \cA$ on event $E_t$,
\begin{align*}
  U_{t, a} - a\T \theta
  = a\T (\hat{\theta}_t - \theta) + C_{t, a}
  \leq C_{t, a} + C_{t, a}
  = 2 C_{t, a}\,.
\end{align*}

\textbf{Case 2: The gap is small, $\Delta_{A_t} \leq \varepsilon$.} Then naively $\Delta_{A_t} \leq \varepsilon$.

\textbf{Case 3: Event $E_t$ does not occur.} Then $\Delta_{A_t} \I{\bar{E}_t} \leq 2 \normw{A_t}{2} \normw{\theta}{2} \I{\bar{E}_t} \leq 2 L L_* \I{\bar{E}_t}$, where $2 L L_*$ is a trivial upper bound on $\Delta_{A_t}$. We bound the event in expectation as follows.

\begin{lemma}
\label{lem:failure probability} For any round $t \in [n]$ and history $H_t$, we have that $\condprob{\bar{E}_t}{H_t} \leq 2 K \delta$.
\end{lemma}
\begin{proof}
First, note that for any history $H_t$,
\begin{align*}
  \condprob{\bar{E}_t}{H_t}
  \leq \sum_{a \in \cA}
  \condprob{|a\T (\theta - \hat{\theta}_t)|
  \geq \sqrt{2 \log(1 / \delta)} \normw{a}{\hat{\Sigma}_t}}{H_t}\,.
\end{align*}
By definition, $\theta - \hat{\theta}_t \mid H_t \sim \cN(\mathbf{0}_d, \hat{\Sigma}_t)$, and therefore $a\T (\theta - \hat{\theta}_t) / \normw{a}{\hat{\Sigma}_t} \mid H_t \sim \cN(0, 1)$ for any action $a \in \cA$. It immediately follows that
\begin{align*}
  \condprob{|a\T (\theta - \hat{\theta}_t)|
  \geq \sqrt{2 \log(1 / \delta)} \normw{a}{\hat{\Sigma}_t}}{H_t}
  \leq 2 \delta\,.
\end{align*}
This completes the proof.
\end{proof}

Finally, we chain all inequalities, add them over all rounds, and get
\begin{align*}
  R(n)
  \leq 8 \E{\frac{1}{\Delta_{\min}^\varepsilon}
  \sum_{t = 1}^n \normw{A_t}{\hat{\Sigma}_t}^2} \log(1 / \delta) +
  \varepsilon n + 4 L L_* K n \delta\,.
\end{align*}
The sum can bounded using a worst-case argument below, which yields our claim.

\begin{lemma}
\label{lem:posterior variances} The sum of posterior variances is bounded as
\begin{align*}
  \sum_{t = 1}^n \normw{A_t}{\hat{\Sigma}_t}^2
  \leq \frac{\sigma_{0, \max}^2 d}{\log\left(1 + \frac{\sigma_{0, \max}^2}{\sigma^2}\right)}
  \log\left(1 + \frac{\sigma_{0, \max}^2 n}{\sigma^2 d}\right)\,.
\end{align*}
\end{lemma}
\begin{proof}
We start with an upper bound on the posterior variance of the mean reward estimate of any action. In any round $t \in [n]$, by Weyl's inequalities, we have
\begin{align*}
  \lambda_1(\hat{\Sigma}_t)
  = \lambda_1((\Sigma_0^{-1} + G_t)^{-1})
  = \lambda_d^{-1}(\Sigma_0^{-1} + G_t)
  \leq \lambda_d^{-1}(\Sigma_0^{-1})
  = \lambda_1(\Sigma_0)\,.
\end{align*}
Thus, when $\normw{a}{2} \leq L$ for any action $a \in \cA$, we have $\max_{a \in \cA} \normw{a}{\hat{\Sigma}_t} \leq \sqrt{\lambda_1(\Sigma_0)} L = \sigma_{0, \max}$.

Now we bound the sum of posterior variances $\sum_{t = 1}^n \normw{A_t}{\hat{\Sigma}_t}^2$. Fix round $t$ and note that
\begin{align}
  \normw{A_t}{\hat{\Sigma}_t}^2
  = \sigma^2 \frac{A_t\T \hat{\Sigma}_t A_t}{\sigma^2}
  \leq c_1 \log(1 + \sigma^{-2} A_t\T \hat{\Sigma}_t A_t)
  & = c_1 \log\det(I_d + \sigma^{-2}
  \hat{\Sigma}_t^\frac{1}{2} A_t A_t\T \hat{\Sigma}_t^\frac{1}{2})
  \label{eq:sequential proof decomposition}
\end{align}
for
\begin{align*}
  c_1
  = \frac{\sigma_{0, \max}^2}{\log(1 + \sigma^{-2} \sigma_{0, \max}^2)}\,.
\end{align*}
This upper bound is derived as follows. For any $x \in [0, u]$,
\begin{align*}
  x
  = \frac{x}{\log(1 + x)} \log(1 + x)
  \leq \left(\max_{x \in [0, u]} \frac{x}{\log(1 + x)}\right) \log(1 + x)
  = \frac{u}{\log(1 + u)} \log(1 + x)\,.
\end{align*}
Then we set $x = \sigma^{-2} A_t\T \hat{\Sigma}_t A_t$ and use the definition of $\sigma_{0, \max}$.

The next step is bounding the logarithmic term in \eqref{eq:sequential proof decomposition}, which can be rewritten as
\begin{align*}
  \log\det(I_d + \sigma^{-2}
  \hat{\Sigma}_t^\frac{1}{2} A_t A_t\T \hat{\Sigma}_t^\frac{1}{2})
  = \log\det(\hat{\Sigma}_t^{-1} + \sigma^{-2} A_t A_t\T) -
  \log\det(\hat{\Sigma}_t^{-1})\,.
\end{align*}
Because of that, when we sum over all rounds, we get telescoping and the total contribution of all terms is at most
\begin{align*}
  \sum_{t = 1}^n
  \log\det(I_d + \sigma^{-2} \hat{\Sigma}_t^\frac{1}{2} A_t
  A_t\T \hat{\Sigma}_t^\frac{1}{2})
  & = \log\det(\hat{\Sigma}_{n + 1}^{-1}) - \log\det(\hat{\Sigma}_1^{-1}) \\
  & = \log\det(\Sigma_0^\frac{1}{2} \hat{\Sigma}_{n + 1}^{-1} \Sigma_0^\frac{1}{2}) \\
  & \leq d \log\left(\frac{1}{d} \trace(\Sigma_0^\frac{1}{2} \hat{\Sigma}_{n + 1}^{-1}
  \Sigma_0^\frac{1}{2})\right) \\
  & = d \log\left(1 + \frac{1}{\sigma^2 d} \sum_{t = 1}^n
  \trace(\Sigma_0^\frac{1}{2} A_t A_t\T \Sigma_0^\frac{1}{2})\right) \\
  & = d \log\left(1 + \frac{1}{\sigma^2 d} \sum_{t = 1}^n
  A_t\T \Sigma_0 A_t\right) \\
  & \leq d \log\left(1 + \frac{\sigma_{0, \max}^2 n}{\sigma^2 d}\right)\,.
\end{align*}
This completes the proof.
\end{proof}

%% file: Corollary.tex
\section{Complete Statement of \cref{thm:prior-dependent bayesucb}}
\label{sec:complete prior-dependent regret bound}

\begin{theorem}
\label{thm:complete prior-dependent regret bound} Let $\sigma_0^2 \geq \frac{1}{8 \log(1 / \delta) \, n^2 \log \log n}$. Then there exist functions $\xi_a: \realset \to \left[\frac{1}{n}, \frac{1}{\sqrt{\log n}}\right]$ such that the $n$-round Bayes regret of \bayesucb in a $K$-armed Gaussian bandit is bounded as
\begin{align*}
  R(n)
  \leq \left[8 \sigma^2 \log(1 / \delta) \log n -
  \frac{\sigma^2}{2 \sigma_0^2 \log n}\right] \sum_{a \in \cA} \int_{\theta_{- a}}
  h_a(\theta_a^* - \xi_a(\theta_a^*)) \, h_{- a}(\theta_{- a}) \dif \theta_{- a} + C\,,
\end{align*}
where $C = 8 \sigma^2 K \log(1 / \delta) \sqrt{\log n} + (2 \sqrt{2 \log(1 / \delta)} + 1) \sigma_0 K n \delta + 1$ is a low-order term.

Moreover, when $\sigma_0^2 < \frac{1}{8 \log(1 / \delta) \, n^2 \log \log n}$, the regret is bounded as
\begin{align*}
  R(n)
  \leq \frac{2 \sqrt{2 \log(1 / \delta)} + 1}
  {\sqrt{8 \log(1 / \delta) \log \log n}} K \delta + 1\,.
\end{align*}
\end{theorem}
\begin{proof}
The first claim is proved in \cref{sec:prior-dependent bayesucb proof}. The second claim can be proved as follows. Take \cref{thm:gap-dependent bayesucb}, set $\varepsilon = 0$, and consider the three cases in \cref{sec:gap-dependent bayesucb proof}.

\textbf{Case 1: Event $E_t$ occurs and the gap is large, $\Delta_{A_t} \geq \varepsilon$.} On event $E_t$, action $a$ can be taken only if
\begin{align*}
  \Delta_a
  \leq 2 \sqrt{\frac{2 \log(1 / \delta)}{\sigma_0^{-2} + \sigma^{-2} N_{t, a}}}
  \leq 2 \sqrt{2 \sigma_0^2 \log(1 / \delta)}
  \leq 2 \sqrt{\frac{1}{4 n^2 \log \log n}}
  < \frac{1}{n}\,.
\end{align*}
Therefore, the corresponding $n$-round regret is bounded by $1$.

\textbf{Case 2: The gap is small, $\Delta_{A_t} < \varepsilon$.} This case cannot happen because $\varepsilon = 0$.

\textbf{Case 3: Event $E_t$ does not occur.} The $n$-round regret is bounded by
\begin{align*}
  (2 \sqrt{2 \log(1 / \delta)} + 1) \sigma_0 K n \delta
  \leq \frac{2 \sqrt{2 \log(1 / \delta)} + 1}
  {\sqrt{8 \log(1 / \delta) \log \log n}} K \delta\,.
\end{align*}

This completes the proof.
\end{proof}

%% file: GapFreeLinUCB.tex
\section{Gap-Free Regret Bound of \bayesucb in Linear Bandit}
\label{sec:gap-free bound}

Let $E_t$ be the event in \eqref{eq:linear confidence interval}. Our proof has three parts.

\noindent\textbf{Case 1: Event $E_t$ occurs and the gap is large, $\Delta_{A_t} \geq \varepsilon$.} Then
\begin{align*}
  \Delta_{A_t}
  = A_*\T \theta - A_t\T \theta
  \leq A_*\T \theta - U_{t, A_*} + U_{t, A_t} - A_t\T \theta 
  \leq U_{t, A_t} - A_t\T \theta
  \leq 2 C_{t, A_t}\,.
\end{align*}
The first inequality holds because $U_{t, A_t} - U_{t, A_*} \geq 0$ by the design of \bayesucb. The second one uses that $A_*\T \theta - U_{t, A_*} \leq 0$. Specifically, for any action $a \in \cA$ on event $E_t$,
\begin{align*}
  a\T \theta - U_{t, a}
  = a\T (\theta - \hat{\theta}_t) - C_{t, a}
  \leq C_{t, a} - C_{t, a}
  = 0\,.
\end{align*}
The last inequality follows from the definition of event $E_t$. Specifically, for any action $a \in \cA$ on event $E_t$,
\begin{align*}
  U_{t, a} - a\T \theta
  = a\T (\hat{\theta}_t - \theta) + C_{t, a}
  \leq C_{t, a} + C_{t, a}
  = 2 C_{t, a}\,.
\end{align*}

\textbf{Cases 2 and 3} are bounded as in \cref{sec:linear bayesucb proof}. Now we chain all inequalities, add them over all rounds, and get
\begin{align*}
  R(n)
  & \leq 2 \E{\sum_{t = 1}^n \normw{A_t}{\hat{\Sigma}_t}} \sqrt{2 \log(1 / \delta)} +
  \varepsilon n + 4 L L_* K n \delta \\
  & \leq 2 \sqrt{\E{\sum_{t = 1}^n \normw{A_t}{\hat{\Sigma}_t}^2}} \sqrt{2 n \log(1 / \delta)} +
  \varepsilon n + 4 L L_* K n \delta\,,
\end{align*}
where the last inequality uses the Cauchy-Schwarz inequality and the concavity of the square root. Finally, the sum $\sum_{t = 1}^n \normw{A_t}{\hat{\Sigma}_t}^2$ is bounded using \cref{lem:posterior variances}. This completes the proof.

%% file: Bakeoff.tex
\section{Comparison of \bayesucb and \ucb}
\label{sec:bakeoff}

\begin{figure}[t]
  \centering
  \includegraphics[width=5.5in]{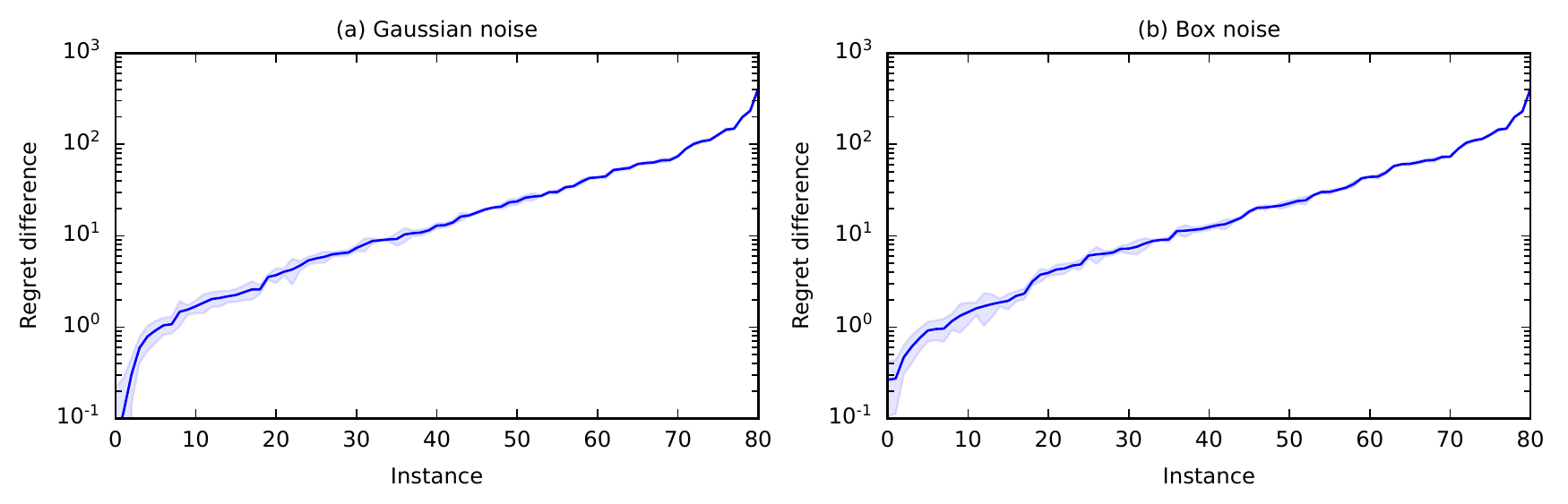}
  \caption{The difference in regret of \ucb and \bayesucb on $81$ Bayesian bandit instances, sorted by the difference. In plot (a), the noise is Gaussian $\cN(0, \sigma^2)$. In plot (b), the noise is $\sigma$ with probability $0.5$ and $- \sigma$ otherwise.}
  \label{fig:bakeoff}
\end{figure}

We report the difference in regret of \ucb and \bayesucb on $81$ Bayesian bandit instances. These instances are obtained by all combinations of $K \in \set{5, 10, 20}$ actions, reward noise $\sigma \in \set{0.5, 1, 2}$, prior gap $\Delta_0 \in \set{0.5, 1, 2}$, and prior width $\sigma_0 \in \set{0.5, 1, 2}$. The horizon is $n = 1\,000$ rounds and all results are averaged over $1\,000$ random runs.

Our results are reported in \cref{fig:bakeoff}. In \cref{fig:bakeoff}a, the noise is Gaussian $\cN(0, \sigma^2)$. In \cref{fig:bakeoff}b, the noise is $\sigma$ with probability $0.5$ and $- \sigma$ otherwise. Therefore, this noise is $\sigma^2$-sub-Gaussian, of the same magnitude as $\cN(0, \sigma^2)$ but far from it in terms of the distribution. This tests the robustness of \bayesucb to Gaussian posterior updates. \ucb only needs $\sigma^2$-sub-Gaussian noise. In both plots, and in all $81$ Bayesian bandit instances, \bayesucb has a lower regret than \ucb. It is also remarkably robust to noise misspecification, although we cannot prove it.

%% file: InfoTheoryNote.tex
\section{Note on Information-Theory Bounds}
\label{sec:information-theory note}

Our approach could be used to derive Bayesian information-theory bounds \citep{russo16information}. The key step in these bounds, where the information-theory term $I_{t, a}$ for action $a$ at round $t$ arises, is $\Delta_{A_t} \leq \Gamma \sqrt{I_{t, A_t}}$, where $\Gamma$ is the highest possible ratio of regret to information gain. As in Case 1 in \cref{sec:linear bayesucb proof}, the $n$-round regret can be bounded as
\begin{align*}
  \sum_{t = 1}^n \Delta_{A_t}
  = \sum_{t = 1}^n \frac{1}{\Delta_{A_t}} \Delta_{A_t}^2
  \leq \frac{1}{\Delta_{\min}^\varepsilon} \sum_{t = 1}^n \Delta_{A_t}^2
  \leq \frac{\Gamma^2}{\Delta_{\min}^\varepsilon} \sum_{t = 1}^n I_{t, A_t}\,.
\end{align*}
The term $\sum_{t = 1}^n I_{t, A_t}$ can be bounded using a worst-case argument by a $O(\log n)$ bound.

%% file: Paper.bbl
\begin{thebibliography}{37}
\providecommand{\natexlab}[1]{#1}
\providecommand{\url}[1]{\texttt{#1}}
\expandafter\ifx\csname urlstyle\endcsname\relax
  \providecommand{\doi}[1]{doi: #1}\else
  \providecommand{\doi}{doi: \begingroup \urlstyle{rm}\Url}\fi

\bibitem[Abbasi-Yadkori et~al.(2011)Abbasi-Yadkori, Pal, and
  Szepesvari]{abbasi-yadkori11improved}
Yasin Abbasi-Yadkori, David Pal, and Csaba Szepesvari.
\newblock Improved algorithms for linear stochastic bandits.
\newblock In \emph{Advances in Neural Information Processing Systems 24}, pages
  2312--2320, 2011.

\bibitem[Abeille and Lazaric(2017)]{abeille17linear}
Marc Abeille and Alessandro Lazaric.
\newblock Linear {Thompson} sampling revisited.
\newblock In \emph{Proceedings of the 20th International Conference on
  Artificial Intelligence and Statistics}, 2017.

\bibitem[Agrawal and Goyal(2012)]{agrawal12analysis}
Shipra Agrawal and Navin Goyal.
\newblock Analysis of {Thompson} sampling for the multi-armed bandit problem.
\newblock In \emph{Proceedings of the 25th Annual Conference on Learning
  Theory}, pages 39.1--39.26, 2012.

\bibitem[Agrawal and Goyal(2013)]{agrawal13thompson}
Shipra Agrawal and Navin Goyal.
\newblock Thompson sampling for contextual bandits with linear payoffs.
\newblock In \emph{Proceedings of the 30th International Conference on Machine
  Learning}, pages 127--135, 2013.

\bibitem[Aouali et~al.(2023)Aouali, Kveton, and Katariya]{aouali23mixedeffect}
Imad Aouali, Branislav Kveton, and Sumeet Katariya.
\newblock Mixed-effect {Thompson} sampling.
\newblock In \emph{Proceedings of the 26th International Conference on
  Artificial Intelligence and Statistics}, 2023.

\bibitem[Audibert and Bubeck(2009)]{audibert09minimax}
Jean-Yves Audibert and Sebastien Bubeck.
\newblock Minimax policies for adversarial and stochastic bandits.
\newblock In \emph{Proceedings of the 22nd Annual Conference on Learning
  Theory}, 2009.

\bibitem[Auer et~al.(1995)Auer, Cesa-Bianchi, Freund, and
  Schapire]{auer95gambling}
Peter Auer, Nicolo Cesa-Bianchi, Yoav Freund, and Robert Schapire.
\newblock Gambling in a rigged casino: The adversarial multi-armed bandit
  problem.
\newblock In \emph{Proceedings of the 36th Annual Symposium on Foundations of
  Computer Science}, pages 322--331, 1995.

\bibitem[Auer et~al.(2002)Auer, Cesa-Bianchi, and Fischer]{auer02finitetime}
Peter Auer, Nicolo Cesa-Bianchi, and Paul Fischer.
\newblock Finite-time analysis of the multiarmed bandit problem.
\newblock \emph{Machine Learning}, 47:\penalty0 235--256, 2002.

\bibitem[Bastani et~al.(2019)Bastani, Simchi-Levi, and Zhu]{bastani19meta}
Hamsa Bastani, David Simchi-Levi, and Ruihao Zhu.
\newblock Meta dynamic pricing: Transfer learning across experiments.
\newblock \emph{CoRR}, abs/1902.10918, 2019.
\newblock URL \url{https://arxiv.org/abs/1902.10918}.

\bibitem[Basu et~al.(2021)Basu, Kveton, Zaheer, and
  Szepesvari]{basu21noregrets}
Soumya Basu, Branislav Kveton, Manzil Zaheer, and Csaba Szepesvari.
\newblock No regrets for learning the prior in bandits.
\newblock In \emph{Advances in Neural Information Processing Systems 34}, 2021.

\bibitem[Bishop(2006)]{bishop06pattern}
Christopher Bishop.
\newblock \emph{Pattern Recognition and Machine Learning}.
\newblock Springer, New York, NY, 2006.

\bibitem[Chapelle and Li(2012)]{chapelle11empirical}
Olivier Chapelle and Lihong Li.
\newblock An empirical evaluation of {Thompson} sampling.
\newblock In \emph{Advances in Neural Information Processing Systems 24}, pages
  2249--2257, 2012.

\bibitem[Dani et~al.(2008)Dani, Hayes, and Kakade]{dani08stochastic}
Varsha Dani, Thomas Hayes, and Sham Kakade.
\newblock Stochastic linear optimization under bandit feedback.
\newblock In \emph{Proceedings of the 21st Annual Conference on Learning
  Theory}, pages 355--366, 2008.

\bibitem[Garivier and Cappe(2011)]{garivier11klucb}
Aurelien Garivier and Olivier Cappe.
\newblock The {KL-UCB} algorithm for bounded stochastic bandits and beyond.
\newblock In \emph{Proceedings of the 24th Annual Conference on Learning
  Theory}, pages 359--376, 2011.

\bibitem[Gittins(1979)]{gittins79bandit}
John Gittins.
\newblock Bandit processes and dynamic allocation indices.
\newblock \emph{Journal of the Royal Statistical Society. Series B
  (Methodological)}, 41:\penalty0 148--177, 1979.

\bibitem[Hong et~al.(2020)Hong, Kveton, Zaheer, Chow, Ahmed, and
  Boutilier]{hong20latent}
Joey Hong, Branislav Kveton, Manzil Zaheer, Yinlam Chow, Amr Ahmed, and Craig
  Boutilier.
\newblock Latent bandits revisited.
\newblock In \emph{Advances in Neural Information Processing Systems 33}, 2020.

\bibitem[Hong et~al.(2022)Hong, Kveton, Zaheer, and
  Ghavamzadeh]{hong22hierarchical}
Joey Hong, Branislav Kveton, Manzil Zaheer, and Mohammad Ghavamzadeh.
\newblock Hierarchical {Bayesian} bandits.
\newblock In \emph{Proceedings of the 25th International Conference on
  Artificial Intelligence and Statistics}, 2022.

\bibitem[Kaufmann et~al.(2012)Kaufmann, Cappe, and
  Garivier]{kaufmann12bayesian}
Emilie Kaufmann, Olivier Cappe, and Aurelien Garivier.
\newblock On {Bayesian} upper confidence bounds for bandit problems.
\newblock In \emph{Proceedings of the 15th International Conference on
  Artificial Intelligence and Statistics}, pages 592--600, 2012.

\bibitem[Kawale et~al.(2015)Kawale, Bui, Kveton, Tran-Thanh, and
  Chawla]{kawale15efficient}
Jaya Kawale, Hung Bui, Branislav Kveton, Long Tran-Thanh, and Sanjay Chawla.
\newblock Efficient {Thompson} sampling for online matrix-factorization
  recommendation.
\newblock In \emph{Advances in Neural Information Processing Systems 28}, pages
  1297--1305, 2015.

\bibitem[Kveton et~al.(2021)Kveton, Konobeev, Zaheer, Hsu, Mladenov, Boutilier,
  and Szepesvari]{kveton21metathompson}
Branislav Kveton, Mikhail Konobeev, Manzil Zaheer, Chih-Wei Hsu, Martin
  Mladenov, Craig Boutilier, and Csaba Szepesvari.
\newblock Meta-{Thompson} sampling.
\newblock In \emph{Proceedings of the 38th International Conference on Machine
  Learning}, 2021.

\bibitem[Lai(1987)]{lai87adaptive}
Tze~Leung Lai.
\newblock Adaptive treatment allocation and the multi-armed bandit problem.
\newblock \emph{The Annals of Statistics}, 15\penalty0 (3):\penalty0
  1091--1114, 1987.

\bibitem[Lai and Robbins(1985)]{lai85asymptotically}
Tze~Leung Lai and Herbert Robbins.
\newblock Asymptotically efficient adaptive allocation rules.
\newblock \emph{Advances in Applied Mathematics}, 6\penalty0 (1):\penalty0
  4--22, 1985.

\bibitem[Lattimore and Szepesvari(2019)]{lattimore19bandit}
Tor Lattimore and Csaba Szepesvari.
\newblock \emph{Bandit Algorithms}.
\newblock Cambridge University Press, 2019.

\bibitem[Li et~al.(2010)Li, Chu, Langford, and Schapire]{li10contextual}
Lihong Li, Wei Chu, John Langford, and Robert Schapire.
\newblock A contextual-bandit approach to personalized news article
  recommendation.
\newblock In \emph{Proceedings of the 19th International Conference on World
  Wide Web}, 2010.

\bibitem[Li et~al.(2018)Li, Jamieson, DeSalvo, Rostamizadeh, and
  Talwalkar]{li18hyperband}
Lisha Li, Kevin Jamieson, Giulia DeSalvo, Afshin Rostamizadeh, and Ameet
  Talwalkar.
\newblock Hyperband: A novel bandit-based approach to hyperparameter
  optimization.
\newblock \emph{Journal of Machine Learning Research}, 18\penalty0
  (185):\penalty0 1--52, 2018.

\bibitem[Li et~al.(2016)Li, Karatzoglou, and Gentile]{li16collaborative}
Shuai Li, Alexandros Karatzoglou, and Claudio Gentile.
\newblock Collaborative filtering bandits.
\newblock In \emph{Proceedings of the 39th Annual International ACM SIGIR
  Conference}, 2016.

\bibitem[Lu and {Van Roy}(2019)]{lu19informationtheoretic}
Xiuyuan Lu and Benjamin {Van Roy}.
\newblock Information-theoretic confidence bounds for reinforcement learning.
\newblock In \emph{Advances in Neural Information Processing Systems 32}, 2019.

\bibitem[Marchal and Arbel(2017)]{marchal17subgaussianity}
Olivier Marchal and Julyan Arbel.
\newblock On the sub-{Gaussianity} of the beta and {Dirichlet} distributions.
\newblock \emph{Electronic Communications in Probability}, 22:\penalty0 1--14,
  2017.

\bibitem[Russo and {Van Roy}(2014)]{russo14learning}
Daniel Russo and Benjamin {Van Roy}.
\newblock Learning to optimize via posterior sampling.
\newblock \emph{Mathematics of Operations Research}, 39\penalty0 (4):\penalty0
  1221--1243, 2014.

\bibitem[Russo and {Van Roy}(2016)]{russo16information}
Daniel Russo and Benjamin {Van Roy}.
\newblock An information-theoretic analysis of {Thompson} sampling.
\newblock \emph{Journal of Machine Learning Research}, 17\penalty0
  (68):\penalty0 1--30, 2016.

\bibitem[Russo et~al.(2018)Russo, {Van Roy}, Kazerouni, Osband, and
  Wen]{russo18tutorial}
Daniel Russo, Benjamin {Van Roy}, Abbas Kazerouni, Ian Osband, and Zheng Wen.
\newblock A tutorial on {Thompson} sampling.
\newblock \emph{Foundations and Trends in Machine Learning}, 11\penalty0
  (1):\penalty0 1--96, 2018.

\bibitem[Simchowitz et~al.(2021)Simchowitz, Tosh, Krishnamurthy, Hsu, Lykouris,
  Dudik, and Schapire]{simchowitz21bayesian}
Max Simchowitz, Christopher Tosh, Akshay Krishnamurthy, Daniel Hsu, Thodoris
  Lykouris, Miro Dudik, and Robert Schapire.
\newblock Bayesian decision-making under misspecified priors with applications
  to meta-learning.
\newblock In \emph{Advances in Neural Information Processing Systems 34}, 2021.

\bibitem[Thompson(1933)]{thompson33likelihood}
William~R. Thompson.
\newblock On the likelihood that one unknown probability exceeds another in
  view of the evidence of two samples.
\newblock \emph{Biometrika}, 25\penalty0 (3-4):\penalty0 285--294, 1933.

\bibitem[Tsitsiklis(1994)]{tsitsiklis94short}
John Tsitsiklis.
\newblock A short proof of the gittins index theorem.
\newblock \emph{Neural Computation}, 4\penalty0 (1):\penalty0 194--199, 1994.

\bibitem[Wagenmaker and Foster(2023)]{wagenmaker23instanceoptimality}
Andrew Wagenmaker and Dylan Foster.
\newblock Instance-optimality in interactive decision making: Toward a
  non-asymptotic theory.
\newblock In \emph{Proceedings of the 36th Annual Conference on Learning
  Theory}, 2023.

\bibitem[Wang et~al.(2021)Wang, Zhang, Singh, Riek, and
  Chaudhuri]{wang21multitask}
Zhi Wang, Chicheng Zhang, Manish~Kumar Singh, Laurel Riek, and Kamalika
  Chaudhuri.
\newblock Multitask bandit learning through heterogeneous feedback aggregation.
\newblock In \emph{Proceedings of the 24th International Conference on
  Artificial Intelligence and Statistics}, 2021.

\bibitem[Zhao et~al.(2013)Zhao, Zhang, and Wang]{zhao13interactive}
Xiaoxue Zhao, Weinan Zhang, and Jun Wang.
\newblock Interactive collaborative filtering.
\newblock In \emph{Proceedings of the 22nd ACM International Conference on
  Information and Knowledge Management}, pages 1411--1420, 2013.

\end{thebibliography}
